\documentclass[12pt]{article}
\usepackage{amsmath}
\usepackage{amsfonts}
\usepackage{enumitem}
\usepackage[title]{appendix}
\usepackage{bm,bbm}
\usepackage{mathtools}
\usepackage[textwidth=8em,textsize=scriptsize]{todonotes}
\usepackage[authoryear, round]{natbib}
\usepackage{amssymb}
\usepackage{amsmath}
\usepackage{thmtools}
\usepackage{comment, float, booktabs}
\usepackage{multirow}
\usepackage{array}
\usepackage{bbding}
\usepackage{color,xcolor}
\RequirePackage[colorlinks,citecolor=blue,urlcolor=blue]{hyperref}
\newlist{enumthm}{enumerate}{1}
\setlist[enumthm]{label=(\alph*)}
\DeclarePairedDelimiter\ceil{\lceil}{\rceil}
\DeclarePairedDelimiter\floor{\lfloor}{\rfloor}
\usepackage{amsthm}
\usepackage[linesnumbered,ruled,vlined]{algorithm2e}

\usepackage{array}
\newcommand{\PreserveBackslash}[1]{\let\temp=\\#1\let\\=\temp}
\newcolumntype{C}[1]{>{\PreserveBackslash\centering}p{#1}}
\usepackage{colortbl}

\theoremstyle{plain}
\newtheorem*{remark}{Remark}
\newtheorem{theorem}{Theorem}[section]
\newtheorem{lemma}[theorem]{Lemma}
\newtheorem{assumption}{Assumption}

\newtheorem{corollary}[theorem]{Corollary}

\numberwithin{equation}{section}

\usepackage{graphicx}
\usepackage{dsfont}

\usepackage{geometry}
\DeclareMathOperator*{\argmax}{argmax}
\DeclareMathOperator*{\argmin}{argmin} 
\def\0{\mathbf{0}}
\def\cB{\mathcal{B}}

\def\cF{\mathcal{F}}

\def\cG{\mathcal{G}}
\def\bI{\mathbf{I}}
\def\cL{\mathcal{L}}
\def\hcL{\widehat{\mathcal{L}}}
\def\cE{\mathcal{E}}

\def\Rbb{\mathbb{R}}
\def\cX{\mathcal{X}}
\def\cY{\mathcal{Y}}
\def\Ebb{\mathbb{E}}
\def\ld{\ldots}
\def\fracsumton{\frac{1}{n}\sum_{i=1}^n}
\def\tht{\theta}
\def\hG{\hat{G}}
\def\hD{\hat{D}}
\def\supcF{\sup_{f\in\cF^1}}

\def\vtheta{{\bm{\theta}}}
\def\hvtheta{\hat{\vtheta}}
\def\vphi{{\bm{\phi}}}

\title{Wasserstein Generative Learning of\\ Conditional Distribution}

\author{
Shiao Liu\thanks{Equal contribution.} \thanks{Department of Statistics and Actuarial Science, University of Iowa, Iowa City, Iowa 52242, USA. Email: shiao-liu@uiowa.edu}
\  Xingyu Zhou$^*$\thanks{Department of Statistics and Actuarial Science, University of Iowa, Iowa City, Iowa 52242, USA. Email: xingyu-zhou@uiowa.edu}
\ Yuling Jiao\thanks{School of Mathematics and Statistics, Wuhan University, Wuhan, Hubei, China 430072.  Email: yulingjiaomath@whu.edu.cn}
 \ and
 Jian Huang\thanks{Department of Statistics and Actuarial Science, University of Iowa, Iowa City, Iowa 52242, USA. Email: jian-huang@uiowa.edu}}

\date{\today}

\begin{document}

\maketitle

\begin{abstract}
Conditional distribution is a fundamental quantity for describing the relationship between a response  and a predictor. We propose a Wasserstein generative approach to learning a conditional distribution. The proposed approach uses a conditional generator to transform a known distribution to the target conditional distribution. The conditional generator is estimated by matching a joint distribution involving the conditional generator and the target joint distribution, using the Wasserstein distance as the discrepancy measure for these joint distributions. We establish  non-asymptotic error bound of the conditional sampling distribution generated by the proposed method and show that it is able to mitigate the curse of dimensionality, assuming that the data distribution is supported on a lower-dimensional set. We conduct numerical experiments to validate proposed method and illustrate its applications to conditional sample generation, nonparametric conditional density estimation, prediction uncertainty quantification, bivariate response data,
image reconstruction and image generation.
image generation and reconstruction.
\end{abstract}


\noindent
\textbf{Keywords:}
Conditional distribution; Generative learning; Neural Networks; Non-asymptotic error bounds; Nonparametric estimation.

\section{Introduction}

Conditional distribution is a fundamental quantity
for measuring how a response variable $Y$ depends on a predictor $X$.
Unlike regression methods that only model certain aspects of the relationship between $Y$ and $X$,  such as the conditional mean, conditional distribution provides a complete description of the relationship.
In this paper, we propose a nonparametric generative approach to learning a conditional distribution. This approach uses a function, which we shall refer to as a conditional generator, that transforms a known reference distribution to the target conditional distribution. The conditional generator is estimated by matching the joint distribution involving the conditional generator and the predictor and the joint distribution of the response and the predictor. We use the Wasserstein distance as the discrepancy measure for matching these joint distributions.

There is a vast literature on nonparametric conditional density estimation.
Many existing methods use smoothing techniques, including kernel smoothing and local polynomials
\citep{rosenblatt1969, scott1992,
fan1996,
hyndman1996estimating,
hall2005, bott2017}.
Basis expansion methods  have also been developed for nonparametric conditional density estimation  
 \citep{
 izbicki2016nonparametric, izbicki2017converting}.
A common feature of these methods is that they seek to estimate the functional form of the conditional density. However, these existing conditional density estimation methods
do not work well with high-dimensional complex data.
In addition, most existing methods only consider the case when the response $Y$ is a scalar and is not applicable to the settings when $Y$ is a high-dimensional response vector.

The proposed
approach is inspired by the {generative adversarial networks}
(GAN) \citep{goodfellow14} and Wasserstein GAN (WGAN)  \citep{arjovsky17}.
These methods were developed to learn high-dimensional (unconditional) distributions nonparametrically. Instead of estimating the functional forms of density functions, GAN and WGAN start from a known reference distribution and
use a function that pushes the reference distribution to the data distribution. In practice, this function, often referred to as a generator,  is usually parameterized using deep neural networks.
In GAN and WGAN, it is only necessary to estimate a single generator function for sampling from a
(unconditional) distribution.
However, to sample from a conditional distribution, the generator function necessarily depends on
the given value of $X$ to be conditioned on. Since it is difficult to estimate a collection of generator functions for all the possible values of $X$, an effective approach is to formulate the generator function as a map from the product space of the spaces of the reference distribution and $X$ to the space of $Y$. The existence of such a map is guaranteed by the noise-outsourcing lemma in probability theory \citep{kall2002}.

Several authors have generalized GANs to the setting of learning a conditional distribution. \cite {mirza2014cgan} proposed
 conditional generative adversarial networks (cGAN).
 Similar to GANs, it solves a two-player minimax game
using an objective function with the same form as that of the original GAN \citep{goodfellow14}.
\cite{kovachki2021conditional} proposed a conditional sampling approach with monotone GANs.
A limitation of this approach is that it requires the dimension of the reference distribution to be the same as the dimension of $Y$. In the high-dimensional settings, this does not allow the exploration of a possible low-dimensional latent structure in the data.
\cite{zjlh2021} proposed a generative approach to conditional sampling based on the noise-outsourcing lemma and distribution matching.
The Kullback-Liebler divergence is used for matching the generator distribution and the data distribution. They established consistency of the conditional sampler with respect to the total variation distance, with the help of Pinsker's inequality that bounds the total variation distance via  the Kullback-Liebler divergence \citep{npe2008}.
However, it is difficult to establish the
convergence rate of the conditional sampling distribution with the Kullback-Liebler divergence.

Although the Kullback-Liebler divergence is an attractive discrepancy measure for distributions, it has some drawbacks in the present setting \citep{arjovsky17}. First, the Kullback-Liebler divergence is a strong divergence measure, for example, it is stronger than the Jensen-Shannon divergence and the total variation distance. Weak convergence of distributions does not imply their convergence in Kullback-Liebler divergence. Second, in high-dimensional problems, we are often interested in learning distributions with a low-dimensional latent structure, whose density functions may not exist. In this case, it is not sensible to use
the Kullback-Liebler divergence. In contrast, the Wasserstein distance metricizes the space of probability distributions under mild conditions. This enables us to obtain the non-asymptotic error bounds of
the proposed method
and its convergence rate, see Theorems \ref{thm1} and \ref{thm2} in Section \ref{errorbound}.
Since the computation of  the proposed method
does not involve density functions, we can use it
to learn distributions without a density function, such as distributions supported on a set with a lower intrinsic dimension than the ambient dimension.  We show that the proposed method has an improved convergence rate under a low-dimensional support assumption and
can mitigate the curse of dimensionality, see Theorem \ref{thmlowdim} in Section \ref{cod}.

The proposed
method for learning a conditional distribution has several appealing properties compared with the standard conditional density estimation methods. First,  the proposed method
can use a reference distribution with a lower dimension than that of the target distribution, therefore, it can learn conditional distributions with a lower-dimensional latent structure by using a low-dimensional reference distribution.
Second,  there is no restriction on the dimensionality of the response variable, while the standard methods typically only consider the case of a scalar response variable. Third, the proposed method allows continuous, discrete and mixed types of predictors and responses, while the smoothing and basis expansion methods only deal with continuous-type variables. Finally, since the computation of Wasserstein distance does not involve density functions, it can be used as a loss function for learning distributions without a density function, such as distributions supported on a set with a lower intrinsic dimension than the ambient dimension. Also, by using the Wassertein distance, we are able to establish the non-asymptotic error bound of the generated sampling distribution. To the best of our knowledge, this is the first error bound result in the context of conditional generative learning.

The remainder of this paper is organized as follows. In Section \ref{wgcs}
 we  present the proposed
  method.
 In Section \ref{convergence} we establish the non-asymptotic error bounds for 
the  generated sampling distribution by the proposed method.
 in terms of the Dudley metric. 
 We also show that the proposed method 
 is able to mitigate the curse of dimensionality under the assumption that the joint distribution of
$(X, Y)$ is supported on a set with a low Minkowski dimension. 
In Section \ref{implementation} we desribe the implementation of the proposed method. 
In Section \ref{experiments}  we conduct numerical experiments to validate the proposed method and illustrate its applications in conditional sample generation,  nonparametric conditional density estimation, visualization of multivariate data, image generation and image reconstruction.
Concluding remarks are given in Section \ref{conclusion}. Additional numerical experiment results and the technical details are given in the appendices.

\section{Method} 
\label{wgcs}

We first describe an approach to representing a conditional distribution via a conditional generator function based
on the noise-outsoucing lemma. We then describe the proposed method
based on matching
distributions using Wasserstein metric.

\subsection{Conditional sampling based on noise outsourcing}
Consider a pair of random vectors  $(X, Y) \in \cX \times \cY$, where $\cX \subseteq \mathbb{R}^d$ and $\cY \subseteq \mathbb{R}^q.$
Suppose $(X, Y)\sim P_{X, Y}$ with marginal distributions
$X \sim P_X$ and $Y\sim P_Y$. Denote the conditional distribution of $Y$ given $X$ by $P_{Y\mid X}$. For a given value $x$ of $X$, we also write the conditional distribution as $P_{Y\mid X=x}$.
For regression problems, $\cY \subseteq \mathbb{R}^q$ with $q \ge 1$;  for classification problems, $\cY$  is a set of finitely many labels.
We assume $\cX \subseteq \mathbb{R}^d$ with $d \ge 1$.
The random vectors $X$ and $Y$ can contain both continuous and categorical components.
Let $\eta \in \mathbb{R}^m$ be a random vector independent of $(X, Y)$ with a known distribution $P_{\eta}$ that is easy to sample from,  such as a normal distribution.

We are interested in finding a function  $G: \mathbb{R}^m \times \cX \mapsto \cY$ such that the conditional distribution of $G(\eta,X)$ given $X=x$ equals the conditional distribution of $Y$ given $X=x$.
Since $\eta$ is independent of $X$, this is equivalent to finding a $G$ such that
\begin{equation}
\label{dm1a}
G(\eta, x) \sim P_{Y\mid X=x},\ x \in \cX.
\end{equation}
Therefore,  to sample from the conditional distribution $P_{Y\mid X=x}$, we can first sample an $\eta\sim P_{\eta}$ and calculate $G(\eta, x)$, then the resulting $G(\eta, x)$  is a sample from
$ P_{Y\mid X=x}.$  Because of this property, we shall refer to $G$ as a conditional generator.
This approach has also been used in \cite{zjlh2021}.
The existence of such a $G$ is guaranteed by the noise-outsourcing lemma
 (Theorem 5.10 in
\cite{kall2002}). For ease of reference, we state it here with a slight modification.

\begin{lemma}
\label{NOLemma}
(Noise-outsourcing lemma).
Suppose $\cY$ is a standard Borel space. Then there exist a random vector $\eta \sim
N(\mathbf{0},  \mathbf{I}_m)$ for a given $m \ge 1$ and a Borel-measurable function $G : \mathbb{R}^m \times \cX \to \cY$ such that $\eta$ is independent of $X$ and
\begin{equation}
\label{NOa}
(X, Y)= (X, G(\eta, X))  \ \text{almost surely.}
\end{equation}
\end{lemma}

In Lemma \ref{NOLemma},  the noise distribution $P_{\eta}$
is taken to be $N(\0,  \bI_m)$ with $m \ge 1$, since it is convenient to generate random numbers from a standard normal distribution.  It is a simple consequence of the original noise-outsourcing lemma which uses the uniform distribution on $[0, 1]$ for the noise distribution.
The dimension  $m$ of the noise vector does not need to be the same as $q$, the dimension of $Y.$

Because $\eta$ and $X$ are independent,
a $G$ satisfies \eqref{dm1a} if and only if it also satisfies  \eqref{NOa}.
Therefore, to construct the conditional generator, we can find a $G$ such that the joint distribution of $(X, G(\eta, X))$ matches the joint distribution of $(X, Y)$. This is the basis of the proposed generative approach described below.

\subsection{Wasserstein generative conditional sampler}
\label{wgcs2}
Let $\Omega$  be a subset of $\mathbb{R}^k, k \ge 1,$ on which the measures we consider are defined. Let $\cB_1(\Omega)$ be the set of Borel probability measures on $\Omega$ with finite first moment, that is,  the set of probability measures
$\mu$ on $\mathbb{R}^k$ such that $\int_{\Omega}\|x\|_1 d\mu(x) < \infty$, where
$\|\cdot\|_1$ denotes the 1-norm in $\mathbb{R}^k$.
The $1$-Wasserstein metric is defined as
\begin{equation}
\label{wassersteinp}
W_1(\mu, \nu)= \inf_{\gamma \in \Gamma(\mu, \nu)}
\int \|u-v\|_1d\gamma(u, v), \  \mu,  \nu \in \cB_1(\Omega),
\end{equation}
where $\Gamma(\mu, \nu)$ is the set of joint probability distributions with marginals
$\mu$ and $\nu.$
The $1$-Wasserstein metric is also known as the
earth mover distance. A computationally more convenient form of the 1-Wasserstein metric is the Monge-Rubinstein dual \citep{villani2008optimal2},
\begin{equation}
\label{1wassersteinb}
W_1(\mu, \nu)=\sup_{f \in \cF^1_{\text{Lip}}}
\left\{\Ebb_{U \sim \mu} f(U) - \Ebb_{V\sim \nu} f(V) \right\}.
\end{equation}
where $\cF^1_{\text{Lip}}$ is the $1$-Lipschitz class,
\begin{equation}
\label{Lpi1}
\cF^1_{\text{Lip}}=\{f: \mathbb{R}^k \to \mathbb{R}, |f(u)-f(v)| \le \|u-v\|_2,\  u, v \in \Omega\}.
\end{equation}
When only random samples from $\mu$ and $\nu$ are available in practice, it is easy to obtain the empirical version of (\ref{1wassersteinb}).

We apply the $1$-Wasserstein metric  (\ref{1wassersteinb}) to the present conditional generative learning problem, that is, we seek a conditional generator function
$G: \mathbb{R}^m \times \mathcal{X} \to \mathbb{R}^q$ satisfying
(\ref{dm1a}). The basic idea is to formulate this problem as a minimisation problem based on the $1$-Wasserstein metric.
By Lemma \ref{NOLemma}, we can find a conditional generator $G$ by matching the joint distributions
$P_{X, G(\eta, X)}$ and $P_{X,Y}.$
By (\ref{1wassersteinb}), the 1-Wasserstein distance between $P_{X, G(\eta, X)}$ and $P_{X,Y}$ is
\[
W_1(P_{X, G(\eta, X)}, P_{X,Y}) = \sup_{D\in \cF^1_{\text{Lip}}}\left\{
\Ebb_{(X, \eta)\sim P_X P_{\eta}} D(X,G(\eta,X))-\Ebb_{(X, Y)\sim P_{X,Y}}
 D(X,Y)\right\}.
\]
We have $W_1 (P_{X, G(\eta, X)},  P_{X, Y}) \ge 0$ for every measurable $G$ and
$W_1(P_{X, G(\eta, X)}, P_{X, Y}) =0$ if and only if $P_{X, G(\eta,X)}=P_{X,Y}.$
Therefore, a sufficient and necessary condition for
\[
G^* \in \argmin_{G \in \cG}  W_1(P_{X, G(\eta, X)}, P_{X,Y}),
\]
is  $P_{X, G^*(\eta, X)} = P_{X, Y}$, which implies
 $G^*(\eta, x) \sim P_{Y\mid X=x}, x \in \cX.$
It follows that the problem of finding the conditional generator can be
formulated as the minimax problem:
\[
\argmin_{G \in \cG} \argmax_{D \in \cF^1_{\text{Lip}}} \cL (G, D),
\]
where
\begin{equation}
\label{objL}
\cL(G,D)=\Ebb D(X,G(\eta,X))-\Ebb D(X,Y).
\end{equation}

Let $\{(X_i,Y_i), i=1,\ldots,n\}$ be a random sample from $P_{X,Y}$ and let $\{\eta_i, i=1,\ldots,n\}$ be independently generated from $P_{\eta}$.
The empirical version of $\cL(G,D)$ based on  $(X_i, Y_i, \eta_i), i=1,\ldots, n,$ is
\begin{align}
\cL_n(G,D) =& \fracsumton D(X_i,G(\eta_i,X_i))-\fracsumton D(X_i,Y_i) .
\label{objE}
\end{align}

We use a feedforward neural network $G_{\vtheta}$ with parameter $\vtheta$ for estimating the conditional generator  G and a second network $D_{\vphi}$ with parameter $\vphi$ for estimating the discriminator $D$.  We refer to
\citet{goodfellow2016deep} for a detailed description of neural network functions.
We estimate  $\vtheta$ and $\vphi$ by solving the minimax problem:
\begin{equation}
(\hat{\vtheta},\hat{\vphi})=\argmin_{\vtheta}\argmax_{\vphi}\cL_n(G_{\vtheta},D_{\vphi}) .
\label{2}
\end{equation}
The estimated conditional generator is $\hG=G_{\hat{\vtheta}}$ and the estimated discriminator is $\hD=D_{\hat{\vphi}}$.

We show below that for $\eta \sim P_{\eta}$,
$\hG(\eta, x)$ converges in distribution to the conditional distribution $P_{Y\mid X=x}, x \in \cX$.
Therefore, we can use $\hG$ to learn this conditional distribution.  Specifically,
for an integer  $J > 1$, we generate a random sample $\{\eta_j, j=1,\ldots, \eta_J\}$ from the reference distribution $P_{\eta}$ and calculate
$\{\hG(\eta_j, x), j=1, \ldots, J\}$, which are approximately distributed as $P_{Y\mid X=x}.$
Since generating random samples from $P_{\eta}$ and calculating $\hG(\eta_j, x)$
are inexpensive,  the computational cost is low once $\hG$ is obtained.
For $Y \in \mathbb{R}$, this immediately leads to the estimated conditional quantiles of $P_{Y\mid X=x}.$
Moreover, for any function $g: \cY \to \mathbb{R}^k$,  by the noise outsourcing lemma, we have
$\Ebb(g(Y)\mid X=x)=\mathbb{E} g(G(\eta, x)).$
We can estimate this conditional expectation by $n^{-1} \sum_{j=1}^J g(\hG(\eta_j;x)).$
In particular, the conditional mean and conditional variance of $P_{Y\mid X=x}$
can be estimated in a straightforward way.
In Section \ref{experiments} below, we illustrate this approach
with a range of  examples.

\section{Non-asymptotic error analysis}
\label{convergence}
We establish the non-asymptotic error bound for the proposed method
in terms of the integral probability metric \citep{muller}
\[
d_{\cF^1_B}(P_{X G}, P_{X,Y}) = \sup_{f \in \cF^1_B} \{\Ebb_{(X, \eta)\sim P_X P_{\eta}}
f(X, G(\eta, X)) - \Ebb_{(X, Y)\sim P_{X,Y}} f(X, Y)\},
\]
where $\cF^1_B$ is the uniformly bounded 1-Lipschitz function class,
\begin{align*}
\mathcal{F}^1_B=\{f: \mathbb{R}^{d+q}\mapsto \mathbb{R},
|f(z_1)-f(z_2)|\leq \|z_1-z_2\| , z_1, z_2 \in  \mathbb{R}^{d+q}\ \text{ and }\|f\|_{\infty}\leq B \}
\end{align*}
for some constant $0 < B < \infty $.
The metric $d_{\cF_B^1}$ is also known as the bounded Lipschitz metric ($d_{\text{BL}}$) which metricizes the weak topology on the space of probability distributions. If $P_{X,Y}$ has a bounded support, then $d_{\text{BL}}$ is essentially the same as the 1-Wasserstein distance.

Let $Z=(X,Y)\sim P_{X,Y}$. We make the following assumptions.

\begin{assumption}\label{asp1}
For some $\delta>0$, $Z$ satisfies the first moment tail condition
\begin{align*}
    \Ebb\|Z\|\mathbbm{1}_{\{\|Z\|>\log t \}}= O(t^{-{(\log t)^{\delta}}/{(d+q)}}), \ \text{ for any }
    t \ge 1.
\end{align*}
\end{assumption}

\begin{assumption}\label{asp2}
The noise distribution $P_{\eta}$ is absolutely continuous with respect to the Lebesgue measure.
\end{assumption}

These are two mild assumptions.
Assumption \ref{asp1}
 is a technical condition for dealing with the case when  the support of $P_{X,Y}$ is an unbounded subset of $\mathbb{R}^{d+q}.$
It can be shown that Assumption \ref{asp1}
is equivalent to $P(\|Z\|>t)=
O(1) ( {1}/{t} ) \exp({-t^{1+\delta}/(d+q)}).$
When $P_{X,Y}$ has a bounded support, Assumption 1 is automatically satisfied.
Moreover, this assumption is satisfied if $P_{X,Y}$ is subgaussian.
Assumption \ref{asp2}
is satisfied by commonly used reference distributions
 such as the normal
 distribution.

For the generator network $G_{\vtheta}$, we require that
\begin{align}
\label{cond1}
\|G_{\vtheta}\|_{\infty}\le \log n.
\end{align}
This requirement is satisfied
by adding an additional clipping layer $\ell$ after the original output layer of the network,
\begin{align*}
    \ell(a)=a\wedge c_{n} \vee (-c_{n})=\sigma(a+c_{n})-\sigma(a-c_{n})-c_{n},
\end{align*}
where $c_{n}=\log n.$
We truncate the value of $\|G_{\theta}\|$ to an increasing cube $[-\log n,\log n]^q$ so that the support of the evaluation function to $[-\log n,\log n]^{d+q}.$ This restricts the evaluation function class to a $2\log n$ domain.

\subsection{Non-asymptotic error bound}
\label{errorbound}

Let $(L_1, W_1)$ be the depth and width of the feedforward ReLU discriminator network $D_{\vphi}$ and
let $(L_2. W_2)$ be the depth and width of the feedforward ReLU generator network $G_{\vtheta}$.
Denote the joint distribution of $(X,\hG(\eta,X))$ by $P_{X, \hG}$.
\begin{theorem}
\label{thm1}
Let  $(L_1, W_1)$ of
 $D_{\vphi}$ and $(L_2, W_2)$ of
 $G_{\vtheta}$ be specified such that
$W_1 L_1=\ceil*{\sqrt{n}}$ and $W_2^2 L_2= c q n $ for some constants $12\leq c\leq 384$.
Then, under Assumptions \ref{asp1} and \ref{asp2},
we have
{
\begin{align*}
    \mathbb{E}_{\hG}d_{\cF_B^1}(P_{X,\hat{G}}, P_{X,Y}) \le
    C (d+q)^{{1}/{2}} n^{-{1}/{(d+q)}}\log n,
\end{align*}
where $C$ is a constant independent of $(n, d, q)$. Here $\mathbb{E}_{\hG}$ represents the expectation with respect to the randomness in $\hG.$
}
\end{theorem}

When $P_{X,Y}$ has a bounded support, we can  drop the logarithm factor in the error bound.
\begin{theorem}
\label{thm2}
Suppose that $P_{X,Y}$ is supported on $[-M,M]^{d+q}$ for $\infty > M>0$.
Let  $(L_1, W_1)$ of $D_{\vphi}$ and $(L_2, W_2)$ of $G_{\vtheta}$ be specified such that
$W_1 L_1=\ceil*{\sqrt{n}}$ and $W_2^2 L_2= c q n $ for some constants $12\leq c\leq 384$.
 Let the output of $G_{\tht}$ be on $[-M,M]^q$. Then, under Assumption \ref{asp2}, we have
{
\begin{align*}
    \mathbb{E}_{\hG}d_{\cF_B^1}(P_{X,\hat{G}}, P_{X,Y})\le
     C ({d+q})^{{1}/{2}} n^{-{1}/{(d+q)}},
\end{align*}
where $C$ is a constant independent of $(n, d, q)$.
}
\end{theorem}

The error bound for the conditional distribution follows as a corollary.

\begin{corollary}
\label{thm3}
Under the same assumptions and conditions of Theorem \ref{thm2}, we have
\begin{align*}
    \mathbb{E}_{\hG}\Ebb_{X} d_{\cF_B^1}(P_{\hat{G}(\eta,X)}, P_{Y\mid X})
    \le
    C ({d+q})^{{1}/{2}} n^{-{1}/{(d+q)}},
\end{align*}
where $C$ is a constant independent of $(n, d, q)$.
\end{corollary}

The proofs  of Theorems \ref{thm1} and \ref{thm2} and Corollary \ref{thm3}  are given in the supplementary material.
Assumption 1 only concerns the first moment tail of the joint distribution of $(X, Y)$. Assumption 2 requires the noise distribution to be absolutely continuous with respect to the Lebesgue measure. These assumptions are easily satisfied in practice. Moreover,  we have made clear how the prefactor in the error bound depends on the dimension $d+q$ of $(X, Y)$. This is useful in the high-dimensional settings since how the prefactor depends on the dimension plays an important role in determining the quality of the bound.
Here the prefactors in the error bounds depend on the dimensions $d$ and $q$ through $(d+q)^{1/2}$. The convergence rate is $n^{-1/(d+q)}$.

Unfortunately, these results suffer from the curse of dimensionality in the sense that
the quality of the error bound deteriorates quickly when $d+q$ becomes large.
It is generally not possible to avoid the curse of dimensionality without any conditions
on the distribution of $(X, Y)$.
Detailed discussions on this problem in the context of nonparametric regression using neural networks can be found in \citet{bauer2019deep, schmidt2020nonparametric} and \citet{jiao2021deep} and the references therein. In particular, \citet{jiao2021deep} also provided an analysis of how the prefactor depends on the dimension of $X$.
To mitigate the curse of dimensionality, certain assumptions on the
distribution of $(X, Y)$ is needed. We consider one such assumption below.

\subsection{Mitigating the curse of dimensionality}
\label{cod}
If the joint distribution of $(X, Y)$ is supported on a low dimensional set, we can improve the convergence rate substantially. Low dimensional structure of  complex data have been frequently observed by researchers in image analysis, computer vision and natural language processing, therefore, it is often a reasonable assumption.
We use the Minkowski dimension as a measure of the dimensionality of a set \citep{bishop_peres2016}.

Let $A$ be a subset of $\Rbb^d$. For any $\varepsilon>0$, the covering number  $\mathcal{N}(\epsilon, A, \|\cdot\|_2)$ is the minimum number of balls with radius $\varepsilon$ needed to cover $A$ \citep{vw1996}.
The upper and the lower Minkowski dimensions of  $A\subseteq \Rbb^d$ are defined respectively as
\begin{align*}
  \dim^{u}_M(A)&=\limsup_{\epsilon\to 0}\frac{\log \mathcal{N}(\epsilon, A, \|\cdot\|_2)}{\log (1/\epsilon)},\\
  \dim^{l}_M(A)&=\liminf_{\epsilon\to 0}\frac{\log \mathcal{N}(\epsilon, A, \|\cdot\|_2)}{\log (1/\epsilon)}.
\end{align*}
If $\dim^{u}_M(A)=\dim^{l}_M(A)=\dim_M(A)$, then $\dim_M(A)$ is called the Minkowski dimension or the metric dimension of $A.$

Minkowski dimension measures how the covering number of the set decays when radius of the covering balls goes to 0. When $A$ is a smooth manifold, its Minkowski dimension equals its own topological dimension characterized by the local homeomorphisms. Minkowski dimension can also be used to measure the dimensionality of highly non-regular set, such as fractals \citep{falconer2004fractal}. \cite{nakada2020adaptive} and \cite{jiao2021deep} have shown
that deep neural networks can adapt to the low dimensional structure of the data and mitigate the curse of dimensionality in nonparametric regression.  We show similar results can be obtained for
the proposed method
if the data distribution is supported on a set with low Minkowski dimension.

\begin{assumption}
\label{aspmin}
Suppose $P_{X,Y}$ is supported on a bounded set $A\subseteq [-M,M]^{d+q}$ and $\dim_M(A)=d_{A}\le d+q$.
\end{assumption}

\begin{theorem}
\label{thmlowdim}
Suppose that $P_{X,Y}$ is supported on $[-M,M]^{d+q}$ for $M>0$.
Let  $(L_1, W_1)$ of $D_{\vphi}$ and $(L_2, W_2)$ of $G_{\vtheta}$ be specified such that
$W_1 L_1=\ceil*{\sqrt{n}}$ and $W_2^2 L_2= c q n $ for some constants $12\leq c\leq 384$.
 Let the output of $G_{\vtheta}$ be on $[-M,M]^q$. Let the output of $G_{\vtheta}$ be on $[-M,M]^q$.
 Then, under Assumptions \ref{asp2} and \ref{aspmin},
we have
\begin{align*}
    \mathbb{E}_{\hG}d_{\cF_B^1}(P_{X,\hat{G}}, P_{X,Y})
    \le C ({d+q} )^{{1}/{2}} n^{-{1}/{d_A}},
\end{align*}
where $C$ is a constant independent of $(n, d, q)$.
\end{theorem}

In comparison with Theorem \ref{thm2}, where the rate of convergence depends on $d+q$, the convergence rate in Theorem \ref{thmlowdim} is determined by the Minkowski dimension $d_A$. Therefore,  the assumption of a low Minkowski dimension on the support of data distribution alleviates the curse of dimensionality.
{
However,  unless we have better approximation error bounds for Lipschitz functions defined on low Minkowski dimensional set using neural network functions,  the prefactor in Theorem \ref{thmlowdim} is still $(d+q)^{1/2}$, that is, even the convergence rate only depends on the Minkowski dimension $d_A$, the prefactor still depends on the ambient dimension $d+q.$}

\section{Implementation}
\label{implementation}

The estimator $(\hat\vtheta, \hat\vphi)$ of the neural network parameter $(\vtheta, \vphi)$
is the solution to the minimax problem
\begin{align}
\label{MinimaxA}
	(\hat\vtheta, \hat \vphi) = \argmin_{\vtheta}\argmax_{\vphi} \frac{1}{n}\sum_{i=1}^n\{
D_\vphi(X_i,G_{\vtheta}(\eta_i, X_i))-
D_\vphi(X_i,Y_i)\}.
\end{align}
We use the gradient penalty algorithm to impose the constraint that the discriminator belongs to the class of 1-Lipschitz functions \citep{gulrajani2017improved}.
The minimax problem (\ref{MinimaxA}) is solved by
updating $\vtheta$ and $\vphi$ alternately as follows:
\begin{itemize}
	\item[(a)] Fix  $\vtheta$,  update the discriminator by maximizing the empirical objective function
	\begin{eqnarray*}
		\hat\vphi = \argmax_{\vphi} \frac{1}{n}\sum_{i=1}^n \big\{
D_\vphi(X_i, \hat G_\vtheta(\eta_i, X_i))-
D_\vphi(X_i,Y_i)
-\lambda (\|\nabla_{(x, y)}D_\vphi(X_i,Y_i)\|_2-1)^2\big\}
	\end{eqnarray*}
 with respect to $\vphi$, where $\nabla_{(x,y)}D_\vphi(X_i,Y_i)$ is the gradient of
 $D_\vphi(x, y)$ with respect to $(x, y)$ evaluated at $(X_i, Y_i).$
  The third term on the right side is the gradients penalty for the Lipschitz condition on the discriminator  \citep{gulrajani2017improved}.

	\item[(b)] Fix  $\vphi$,  update the generator by minimizing the empirical objective function
	 \begin{eqnarray*}
		\hat\vtheta = \argmin_{\vtheta}\frac{1}{n}\sum_{i=1}^n
\hat D_\vphi(X_i,G_\vtheta(\eta_i, X_i))
	\end{eqnarray*} with respect to $\vtheta$.
\end{itemize}
We implemented this algorithm in TensorFlow \citep{abadi2016tensorflow}.

We also implemented the weight clipping method for enforcing the Lipschitz condition on the discriminator \citep{arjovsky17}.  With weight clipping, all weights in the discriminator is truncated to be between $[-c,c],$ where $c>0$  is a small number.
We found that the gradient penalty method is more stable and converges faster.  So we only report the numerical results based on the gradient penalty method below.

\section{Numerical experiments}
\label{experiments}

We carry out numerical experiments to evaluate the finite sample performance of
the proposed method described in Section \ref{wgcs}
and illustrate its applications using examples in conditional sample generation,  nonparametric conditional density estimation, visualization of multivariate data, and image generation and reconstruction.
Additional numerical results are provided in the supplementary material.
For all the experiments below, the noise random vector $\eta$ is generated from a standard multivariate normal distribution.
We implemented
the proposed method in
TensorFlow \citep{abadi2016tensorflow} and used the stochastic gradient descent algorithm Adam \citep{kingma2015adam} in training the neural networks.

\subsection{Conditional sample generation: the two-moon dataset}
We use the two-moon data example to illustrate the use of the proposed method for
generating conditional samples.
Let $X\in \{1,2\}$ be the class label and let $Y\in \mathbb{R}^2$. The two-moon model is
\begin{align}
\label{2moon}
Y = \begin{cases}
(\cos(\alpha) + \frac{1}{2} + \epsilon_1, \sin(\alpha) - \frac{1}{6} + \epsilon_2) ,& \text{if $X=1$,} \\
(\cos(\alpha) - \frac{1}{2} + \epsilon_3, -\sin(\alpha) + \frac{1}{6} + \epsilon_4), & \text{if $X=2$,}
\end{cases}
\end{align}
\text{where } $ \alpha\sim \text{Uniform}[0,\pi]$, $\epsilon_1,\ldots, \epsilon_4$ are independent and identically distributed as $N(0,\sigma^2).$
We generate three sets of random samples of size $n=5,000$ with $2,500$ for each class
and $\sigma=0.1, 0.2$ and 0.3 from this model. The generated datasets are shown in the first row of Figure \ref{fig:twomoon}.

 \begin{figure}[H]
	\centering
	\includegraphics[width=4.0in, height=2.0 in]{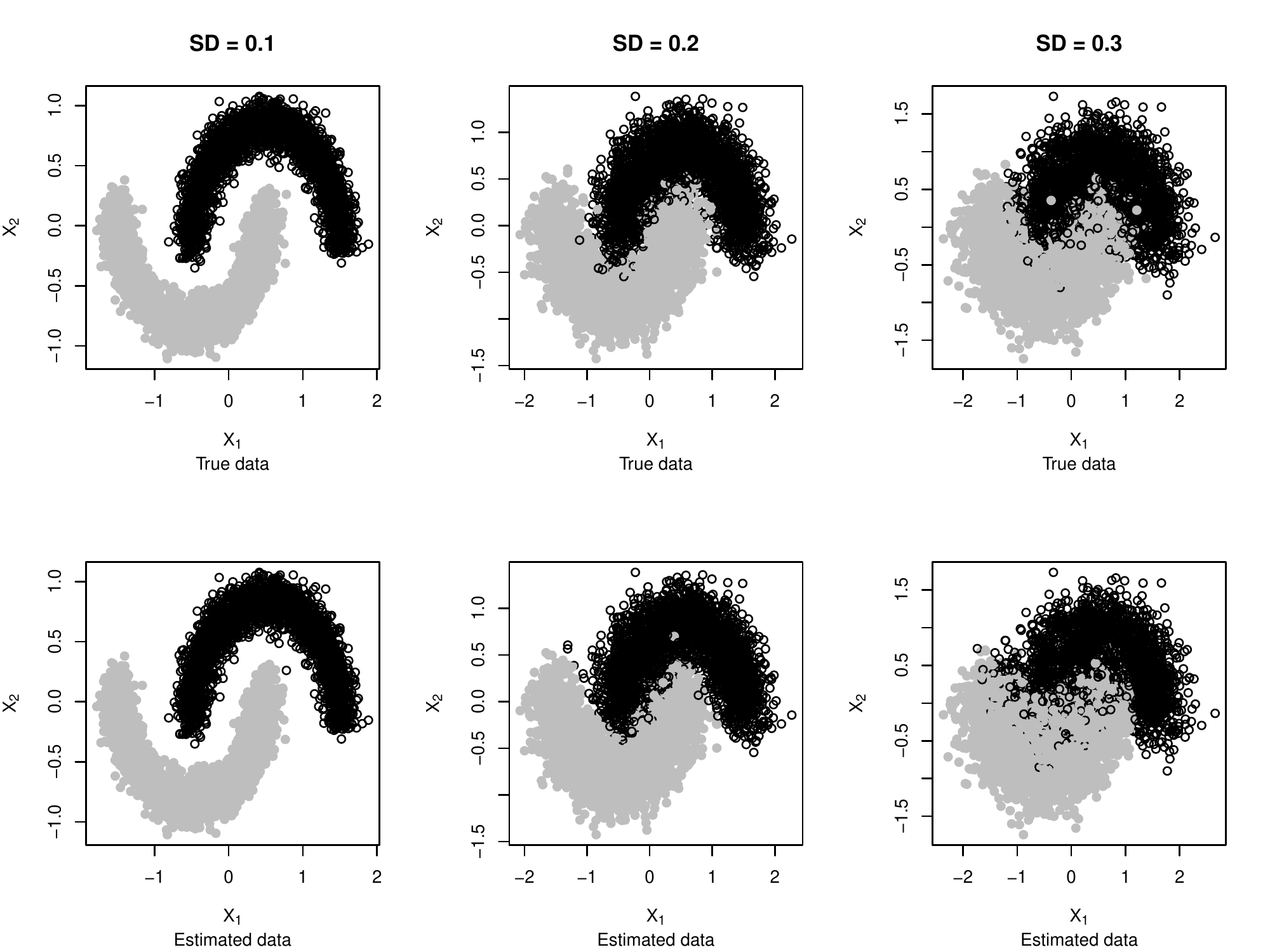}
	\caption{
Comparison of simulated data from (\ref{2moon}) and estimated sample using the proposed method.
From the left to the right, each columns corresponds to a different value of the standard deviation
in the model with $\sigma=0.1, 0.2$ and 0.3, respectively. The first row displays true data and the second row displays the estimated samples.}
	\label{fig:twomoon}
\end{figure}

We use the simulated data to estimate the conditional generator, which is parameterized as
a two-layer fully-connected  network with 30 and 20 nodes.
The discriminator is also a two-layer fully-connected  network with 40 and 20 nodes.
The noise $\eta \sim N(\mathbf{0}, \mathbf{I}_2).$
The activation function for the hidden layer of the generator and the discriminator is ReLU.
For the output layer of the generator, the activation function is the hyperbolic tangent function.
The estimated samples  $\{\hat G(\eta_j, x), j=1. \ldots, 5,000\}$, $x \in \{1, 2\},$ are shown in the second row of Figure \ref{fig:twomoon}.  It can be seen that the scatter plots of the estimated samples are similar to those of the simulated data. This provides a visual validation of the estimated samples in this toy example.

\subsection{Nonparametric conditional density estimation}
We consider the problem of estimating conditional mean and conditional standard deviation in
nonparametric conditional density models. We also compare
the proposed Wasserstein generative conditional sampling method, referred to as
WGCS in Table \ref{tab:point2},
with three existing conditional density estimation methods, including the
nearest neighbor kernel conditional density estimator (NNKCDE) \citep{dalmasso2020conditional}, the conditional kernel density estimator (CKDE, implemented in \texttt{R} package \texttt{np} \citep{hall2004cross},  and a basis expansion method FlexCode \citep{izbicki2017converting}.
We simulated data from the following three models:

\begin{itemize}
\item Model  1 (M1). A nonlinear model with an additive error term:
\begin{align*}Y =  X_1^2 + \exp(X_2 + X_3/3) + \sin(X_4 + X_5) + \varepsilon,\  \varepsilon\sim N(0,1).
\end{align*}
\item Model 2 (M2). 
A model with a multiplicative non-Gassisan error term:
\begin{align*}
Y =(5 +X_1^2/3 + X_2^2 + X_3^2+X_4+X_5) *\exp(0.5\times\varepsilon),
\end{align*}
where $\varepsilon\sim 0.5N(-2,1)+0.5 N(2,1).$
\item Model 3 (M3).
A mixture of two normal distributions:
\begin{align*}Y = 
\mathbb{I}_{\{U\le {1}/{3}\}}N(-1 - X_1 - 0.5X_2, 0.5^2) +
\mathbb{I}_{\{U>{1}/{3}\}}N( 1 + X_1 + 0.5X_2, 1),
\end{align*}
where $U \sim \text{Unif}(0,1)$ and is independent of $X$.
\end{itemize}
In each model, the covariate vector  $X$ is generated from $N(\mathbf{0}, \mathbf{I}_{100})$.
So the ambient dimension of $X$ is 100,  but (M1) and (M2) only depend on the first 5 components of $X$ and (M3) only depends on the first 2 components of $X$.  The sample size $n=5,000$.

For the proposed method, the conditional generator $G$ is parameterized using a
{one-layer} neural network in models (M1) and (M2); it is parameterized by a two-layer fully connected neural network in (M3). The discriminator $D$ is parameterized using a
{two-layer} fully connected neural network. The noise vector $\eta \sim N(0, 1).$

For the conditional density estimation method NNKCDE, the tuning parameters are chosen using cross-validation. The bandwidth of the conditional kernel density estimator CKDE is determined
based on the standard
formula $h_j = 1.06\sigma_j n^{-1/(2k+d)},$ where $\sigma_j$ is a
measure of spread of the $j$th
variable,
$k$ the order of the kernel, and $d$ the dimension of $X$.
The basis expansion based method FlexCode uses 
Fourier basis. The maximum number of bases is set to 40 and the actual number of bases is selected using cross-validation.

\begin{table}[H]
	\centering
	\begin{tabular}{crrrrr} \hline
& &  WGCS &  NNKCDE & CKDE & FlexCode \\ \hline
		\multirow{2}{*}{M1}&Mean &  \textbf{1.10}(0.05) &  2.49(0.01) & 3.30(0.02) & 2.30(0.01) \\
		&SD &  \textbf{0.24}(0.04) & 0.43(0.01) & 0.59(0.01) & 1.06(0.08)\\ \hline
		\multirow{2}{*}{M2}&Mean & \textbf{3.71}(0.23) & 6.09(0.07) & 66.76(2.06) & 10.20(0.33)\\
		&SD & \textbf{3.52}(0.17) &9.33(0.23) & 18.87(0.59) & 11.08(0.34)\\ \hline
		&Mean & 0.32(0.03) & \textbf{0.11}(0.01) & 1.55(0.03) & 0.12(0.04) \\
		\multirow{-2}{*}{M3}&SD & \textbf{0.10}(0.01) &  0.36(0.01) & 0.51(0.01) & 0.33(0.01)\\ \hline
\end{tabular}\\

\medskip
\caption{Mean squared error (MSE) of the estimated conditional mean, the estimated standard deviation and the corresponding simulation standard errors (in parentheses).
The bold numbers indicate the smallest MSEs.
NNKCDE: nearest neighbor kernel conditional density estimator;
CKDE: conditional kernel density estimator;  FlexCode: basis expansion method;
WGCS: Wasserstein generative conditional sampler.
}
	\label{tab:point2}
\end{table}

We consider the mean squared error (MSE) of the estimated conditional mean $\mathbb{E}(Y|X)$ and the estimated conditional standard deviation $\text{SD}(Y|X)$.
We use a test data set $\{x_1,\dots, x_{K}\}$ of size $K=2,000$ .
For the proposed method, we first generate $J=10,000$ samples $\{\eta_j:j=1, \ldots, J\}$ from the reference distribution $P_{\eta}$ and calculate conditional samples  $\{\hat G(\eta_j, x_k), j = 1,\dots, J, k=1, \ldots, K\}$
The estimated conditional standard deviation is calculated as the sample standard deviation of the conditional samples. The MSE of the estimated conditional mean is
$
\text{MSE(mean)}= (1/K)\sum_{k=1}^{K}\{\hat{\mathbb{E}}(Y|X=x_k) - \mathbb{E}(Y|X=x_k)\}^2.
$
For the proposed method,
the estimate of the conditional mean 
is obtained
using Monte Carlo. For other methods, the estimate is calculated by numerical integration.
Similarly, the MSE of the estimated conditional standard deviation is
$ \text{MSE(sd)} = (1/K)\sum_{k=1}^{K}\{\hat{\text{SD}}(Y|X=x_k) -\text{SD}(Y|X=x_k)\}^2.
$
The estimated conditional standard deviation of other methods are computed by numerical integration.

We repeat the simulations 10 times.  The average MSEs and simulation standard errors are summarized in Table \ref{tab:point2}. Comparing with CKDE and FlexCode and NNKCED. WGCS has the smallest MSEs for estimating conditional mean and conditional SD in most cases.

\subsection{Prediction interval: the wine quality dataset}
We use the wine quality dataset \citep{cortez200} to illustrate the application of the proposed method to prediction interval construction.
This dataset is available at UCI machine learning repository
(\url{https://archive.ics.uci.edu/ml/datasets/Wine+Quality}).
There are eleven physicochemical quantitative variables including \textit{fixed acidity, volatile acidity, citric acid, residual sugar, chlorides, free sulfur dioxide, total sulfur dioxide, density, pH, sulphates, alcohol} as predictors and a sensory quantitative variable \textit{quality} that measures wine quality, which is a score between 0 and 10. The goal is to build a prediction model for the wine \textit{quality} score based on the eleven physicochemical variables. Such a model can be used to help the oenologist wine tasting evaluations and improve wine production.
We use the proposed method to learn the conditional distribution
of \textit{quality} given the eleven variables. An advantage of the proposed method over the standard nonparametric regression is that it provides a prediction interval for the quality score, not just a point
prediction.

The sample size of this dataset is 4,898. We use 90\% of the samples as the training set  and 10\% as the test set. All variables are centered and standardized first before training the model.
In our proposed method, the conditional generator is a two-layer fully-connected feedforward network with 50 and 20 nodes,
the discriminator is a two-layer network with 50 and 25 nodes. The ReLU activation function is used in both networks. The noise $\eta\sim N(\mathbf{0}, \mathbf{I}_3).$

\begin{figure}[H]
\centering
\includegraphics[width=5.0 in, height=1.8 in]{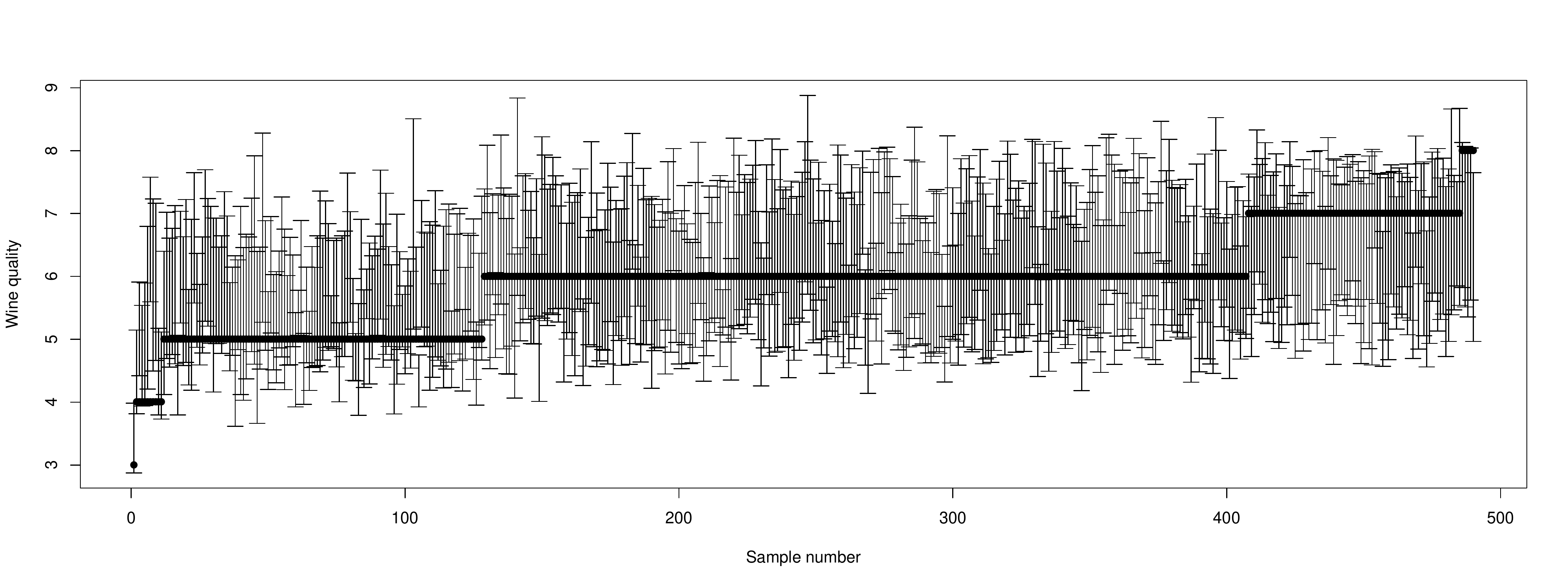}
	\caption{The prediction intervals for the test set.  The actual \textit{quality} scores of 100 randomly selected wine samples in the test set
are shown as the solid dots.}
	\label{fig:wq1}
\end{figure}
To examine the prediction performance of the proposed method, we construct the 90\% prediction interval for the wine quality score in the test set. The prediction intervals are shown in Figure \ref{fig:wq1}. The actual quality scores are plotted as solid dots. The actual coverage for all 490  in the test set is 90.80\%, close to the nominal level of 90\%.

\subsection{
Bivariate response: California housing data}
We use the California housing dataset to demonstrate that the generated conditional samples from the proposed method can be used for visualizing conditional distributions of bivariate response data.
This dataset is available at StatLib repository
(\url{http://lib.stat.cmu.edu/datasets/}).
 It contains 20,640 observations on housing prices with 9 covariates including \textit{median house value, median income, median house age, total rooms, total bedrooms, population, households, latitude, and longitude}. Each row of the data matrix represents a neighborhood block in California.
We use the logarithmic transformation of the variables except
 \textit{longitude, latitude} and \textit{median house age}.
All  columns of the data matrix are centered and standardized to have zero mean and unit variance.
The generator is a two-layer fully-connected neural network with 60 and 25 nodes
and the discriminator is a two-layer network with 65 and 30 nodes. ReLU activation is used for the hidden layers. The random noise vector $\eta\sim N(\mathbf{0}, \mathbf{I}_4)$.

We train the WGCS model with (\textit{median income, median house value}) $\in  \mathbb{R}^2$
as the response vector and other variables as the predictors in the model. We use 18,675 observations (about the 90\% of the dataset) as training set and the remaining 2,064 (about 10\% of the dataset)
observations as the test set.

 \begin{figure}[H]
	\centering
	\includegraphics[width=5.0 in, height=1.0 in]{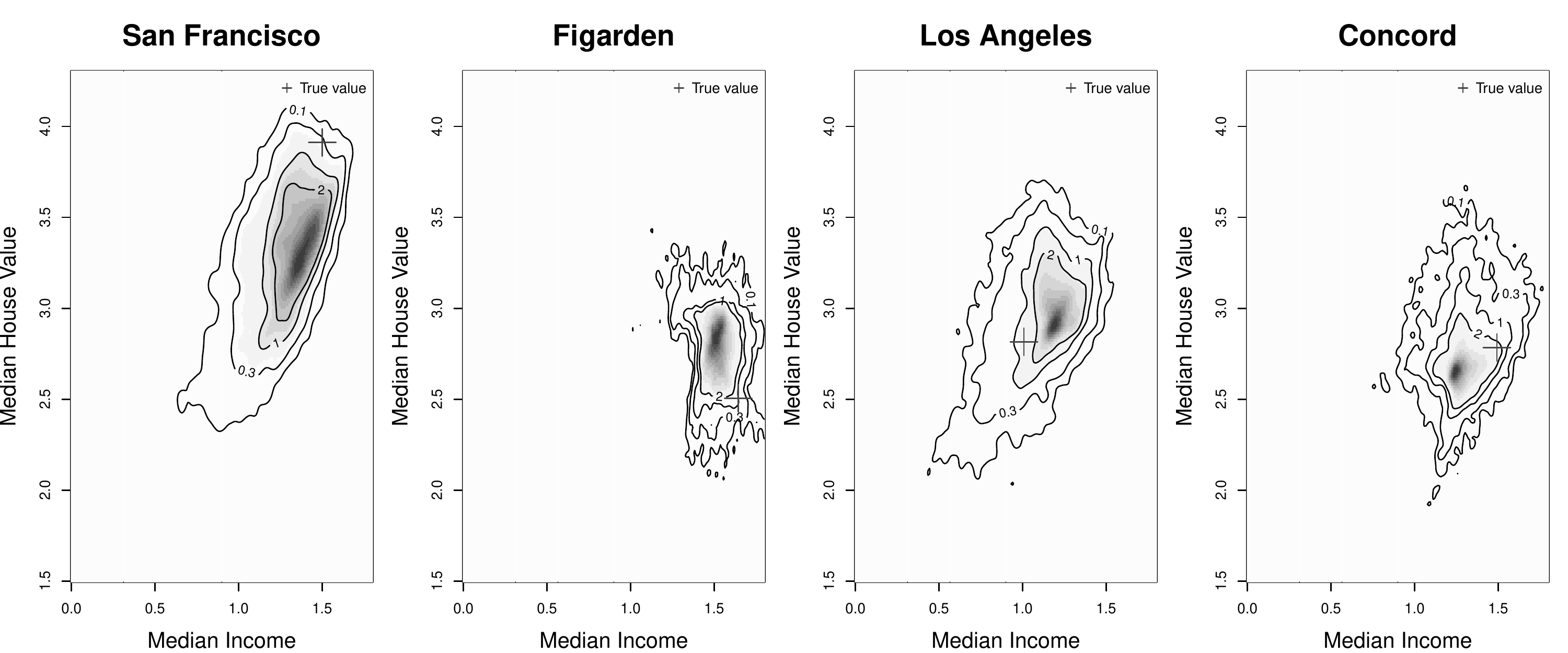}\\
\includegraphics[width= 5.0 in, height=1.0 in]{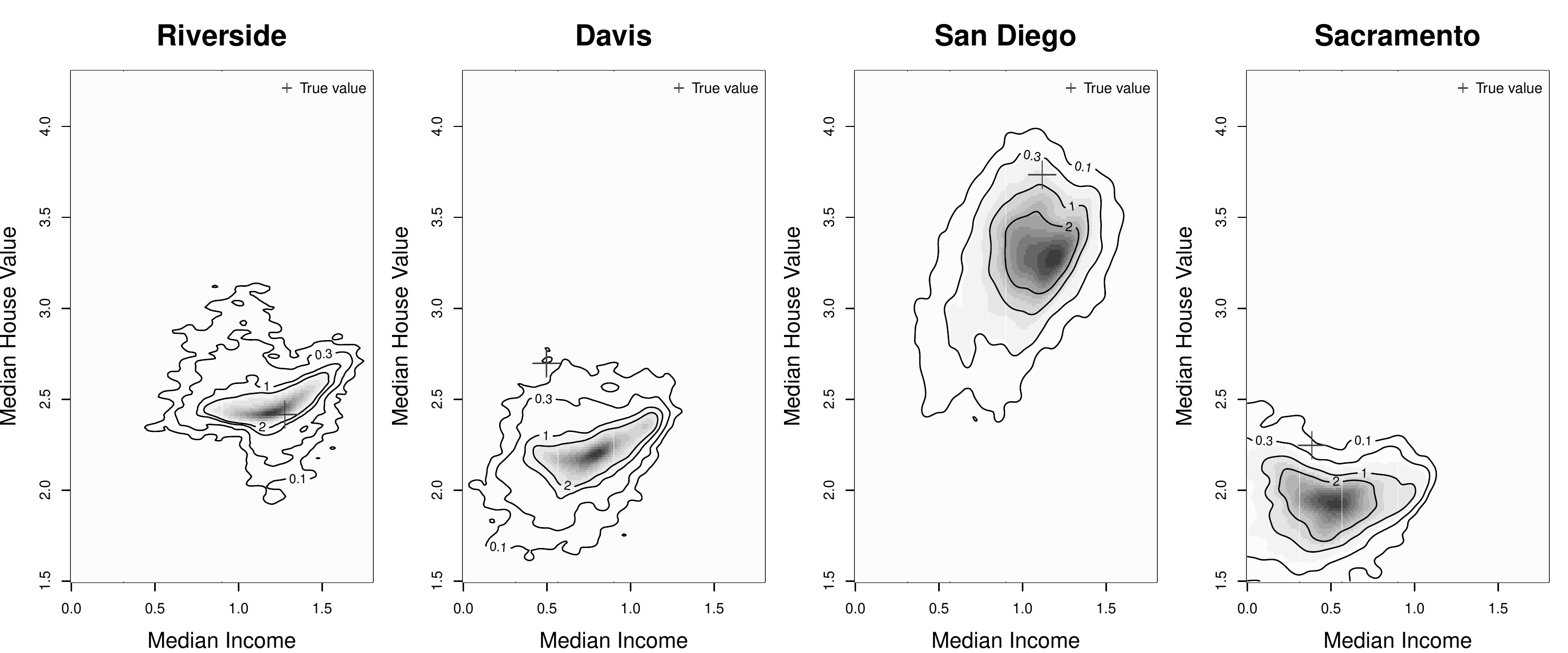}
	\caption{Kernel density estimates of
generated conditional samples of median income and median house value (in logarithmic scale) for eight neighborhood blocks in the California Housing dataset.
The plus sign ``$+$''  represents the observed median income and median house values.}
	\label{fig:calsingle}
\end{figure}

Figure \ref{fig:calsingle} shows the contour plot of the conditional distributions of median income and median house value (in logarithmic scale) for 8 single blocks in the test set. The colored density function represents the kernel density estimates based on the samples generated using the proposed method.
We see that the blue cross is in the main body of the plot, which shows that the proposed method
 provides reasonable prediction of the median income and house values.
We also see that big cities like San Francisco and Los Angles have higher house values with larger variations, smaller cities such as Davis and Concord tend to have lower house values.

\subsection{Image reconstruction: MNIST dataset}
We now illustrate the application of
the proposed method
 to high-dimensional data problems and demonstrate that it can easily handle the models when either of both of $X$ and $Y$ are high-dimensional.
We use
the MNIST handwritten digits dataset \citep{mnist2010},  which contains 60,000 images for training and 10,000 images for testing. The images are stored in $28\times28$ matrices with gray color intensity from 0 to 1. Each image is paired with a label in $\{0,1\dots,9\}$.
We use WGCS to perform two tasks: generating images from labels and reconstructing the missing part of an image.

We  illustrate using the proposed method
for image reconstruction when part of the image is missing with the MNIST dataset.
Suppose we only observe ${1}/{4}, {1}/{2}$ or ${3}/{4}$ of an image and would like to reconstruct the missing part of the image.
For this problem, let $X$ be the observed part of the image
and let $Y$ be the missing part of the image. Our goal is to reconstruct $Y$ based on $X$.
For discriminator, we use two convolutional networks to process  $X$ and $Y$ separately. The filters are then concatenated together and fed into another convolution layer and fully-connected layer before output. For the generator, $X$ is processed by a fully-connected layer followed by 3 deconvolution layers. The random noise vector $\eta \sim N(\mathbf{0}, \mathbf{I}_{100}).$

\begin{figure}[H]
	\centering
	\includegraphics[width=5.0in, height=1.8 in]{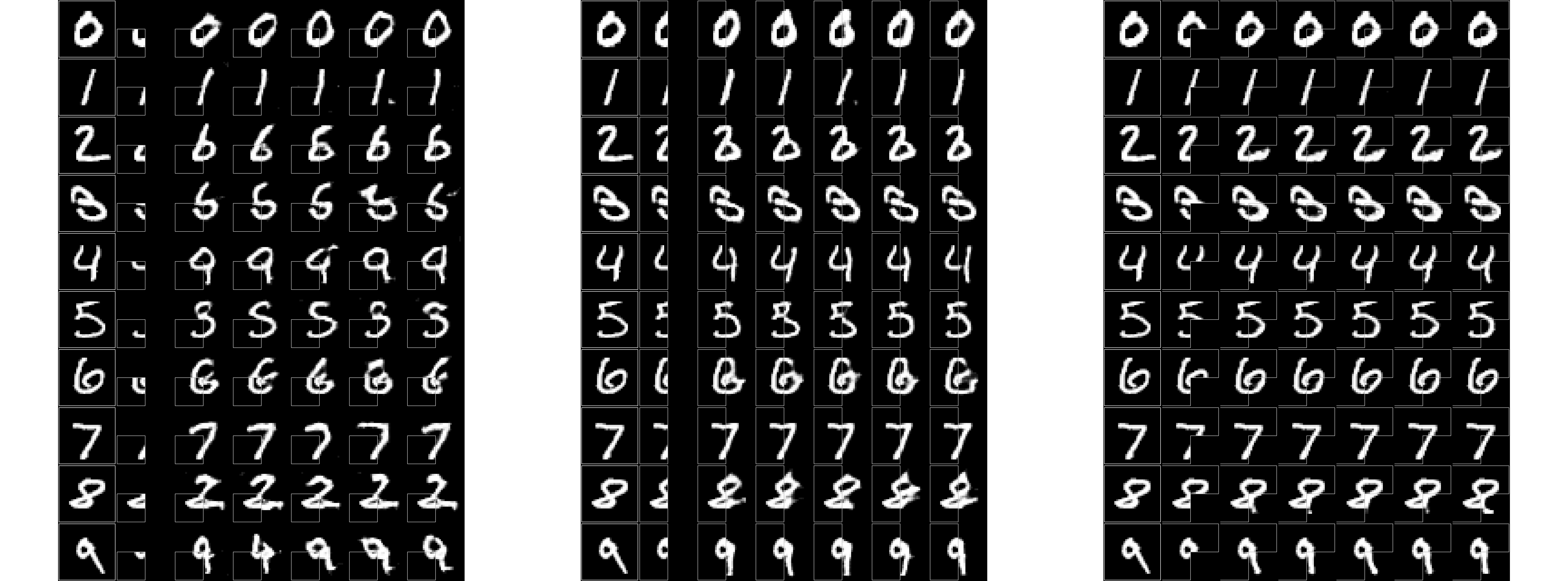}
	\caption{Reconstructed images given partial image in MNIST dataset. The first column in each panel consists of the true images, the other columns give the constructed images.
In the left panel, the left lower 1/4 of the image is given;  in the middle panel, the left 1/2 of the image is given; in the right panel, 3/4 of the image is given.}
	\label{fig:mnist_ximage}
\end{figure}

In Figure \ref{fig:mnist_ximage}, three panels from left to right correspond to the situations when ${1}/{4}, {1}/{2}$ or ${3}/{4}$ of an image are given, the problem is to reconstruct the missing part of the image. In each panel, the first column
contains the true images in the testing set; the second column shows the observed part of the image; the third to the seventh columns show the reconstructed images.
The digits ``0'', ``1'' and ``7'' are easy to reconstruct, even when only 1/4 of their images are given.
For the other digits, if only 1/4 of their images are given, it is difficult to reconstruct them.  However, as the given part increases from 1/4 to 1/2 and then 3/4 of the images,
the proposed method is able to reconstruct the images correctly, and the reconstructed images become less variable and more similar to the true image. For example, for the digit ``2'', if only the left lower 1/4 of the image is given, the reconstructed images tend to be incorrect; the reconstruction is only successful when 3/4 of the image is given.

\subsection{Image generation: CelebA dataset}
We illustrate the application of
the proposed method
to the problem of image generation given label information
with the CelebFaces Attributes (CelebA) dataset \citep{liu2015faceattributes}. This dataset contains more then 200K colored celebrity images with 40 attributes annotation. Here we use 10 binary features, including: \textit{Gender, Young, Bald, Bangs, Blackhair, Blondhair, Eyeglasses, Mustache, Smiling, WearingNecktie}. We  take these features as $X$. So $X$ is a vector consisting of 10 binary components. We consider six types of images; the attributes of these six types are shown in Table \ref{tab:celeba_type}. We used the aligned and cropped images and further resize them to $96\times96$ as our training set. We take these colored images as the values of $Y$. Therefore, the dimension of $X$ is 10 and the dimension of $Y$ is $96\times96\times3=27,648$.

\begin{table}[H]
	\centering
	\begin{tabular}
{|c|C{1.0cm}C{0.8cm}C{1.0cm}C{1.0cm}C{0.8cm}C{0.8cm}C{0.7cm}C{0.7cm}C{1cm}C{1cm}|}
		\hline
		Attributes & Gender & Young&Blkhair&Bldhair&Glass&Bald&Mus&Smile& Necktie&Bangs\\
		\hline
		Type 1 &F &Y&N&N&N&N&N&Y&N&N\\
		Type 2 &F&Y&N&Y&N&N&N&Y&N& N\\
		Type 3 &F&Y&Y&N&N&N&N&N&N& N\\
		Type 4 &M&Y&N&N&N&N &N&N&N&N\\
		Type 5 &M&N&N&N&N&N&N&Y &N&N\\
		Type 6 &M&Y&Y&N&N&N&N&Y&N&N\\
		\hline
	\end{tabular}
	\caption{Attributes for six types of face images in the CelebA dataset:  F$=$Female,  M$=$Male; Y$=$ Yes, N$=$ No.
}
	\label{tab:celeba_type}
\end{table}

The architecture of the discriminator is specified as follows: the 10 dimensional one-hot vector is first expanded to $96\times96\times 10$ by replicating each number to a $96\times 96 $matrix. Then it is concatenated  with image $Y$ on the last axis. Thus the processed data has dimension $96\times96\times 13$. This processed data is then fed into 5 consecutive convolution layers initialized by truncated normal distribution with $SD = 0.01$ with 64,128,256,512,1024 filters respectively. The activation for each convolution layer is a leaky ReLU with $\alpha = 0.2$. The strides is 2 and the kernel size is 5. After convolution, it is flattened and connected to a dense layer with one unit as output.

The architecture of the conditional generator is as follows: the noise vector and the feature vector $X$ are concatenated and fed to a dense layer with $3\times3\times 1024$ units with ReLU activation and then reshaped to $3\times3\times 1024$. Then it goes through 5 layers of deconvolution initialized by truncated normal distribution with $SD = 0.01$ with 512,256,128,64,3 filters respectively. The activation for each intermediate convolution layer is ReLU and hyperbolic tangent function for the last layer. The strides is 2 and kernel size is 5.
The optimizer is Adam with $\beta_1 = 0.5$ and learning rate decrease from $0.0001$ to $0.000005$. The noise dimension is 100.

\begin{figure}[H]
	\centering
	\includegraphics[width=5.0 in, height=2.4 in]{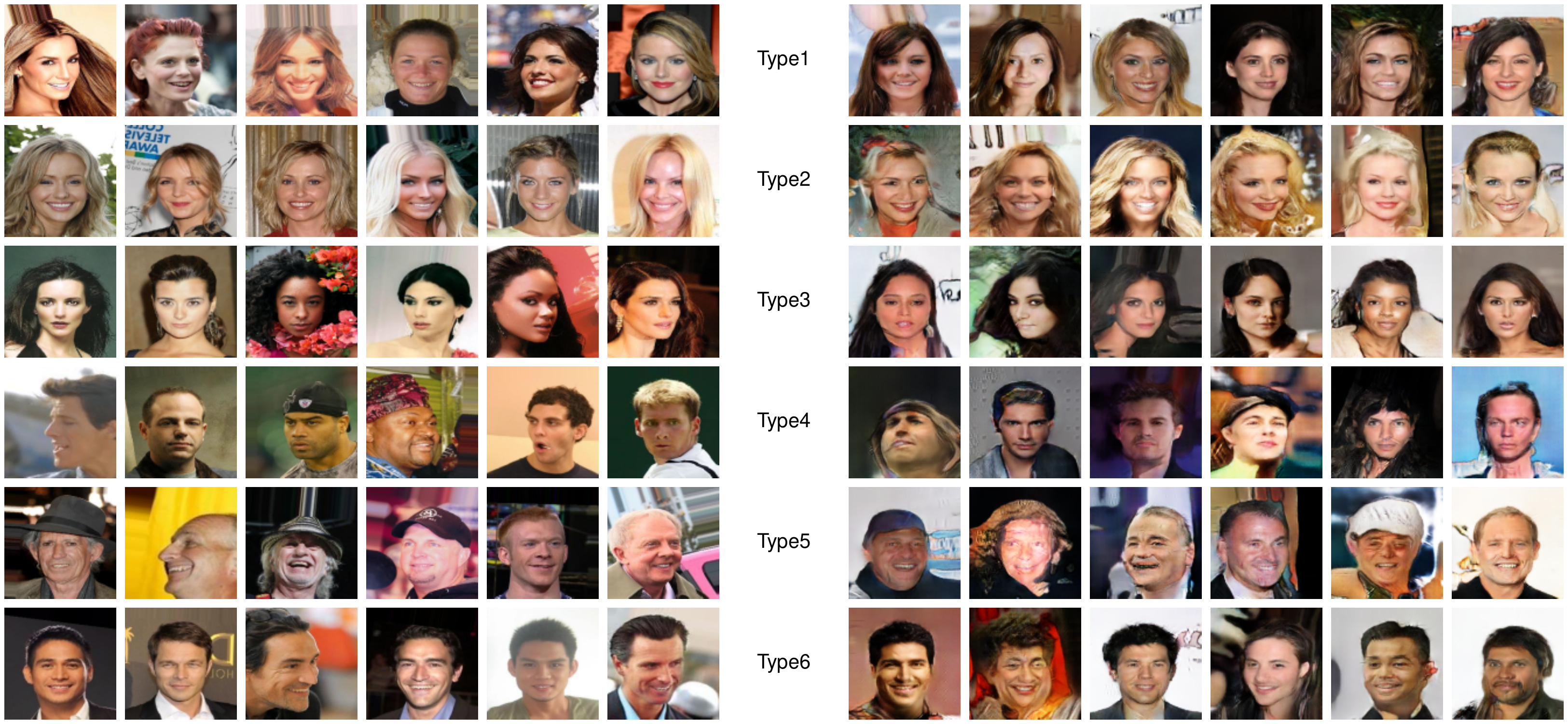}
	\caption{The left panel shows the true images in CelebA dataset. The right panel consists of generated images. Each row corresponds to a specific type of faces.}
	\label{fig:celebatypes}
\end{figure}

In Figure \ref{fig:celebatypes}, the left panel shows real images and  the right panel shows  generated images. Each row corresponds to a specific type of face.
The attributes of the six types of face images are given in Table \ref{tab:celeba_type}.
We can see that the generated images have similar qualities as the original ones.

\section{Discussion}
\label{conclusion}
We have proposed a Wasserstein conditional sampler, a generative approach to learning a conditional distribution. We establish the convergence rate of the sampling distribution of the proposed method.
We also show that the curse of dimensionality can be mitigated if  the data distribution is supported on a lower dimensional set. Our numerical experiments demonstrate that the proposed method performs well in a variety of settings from
standard 
conditional density estimation to more complex image generation and reconstruction problems.

In the future work, it would be interesting to consider incorporating additional information such as sparsity and latent low dimensional structure of data
in the proposed method to better deal with the curse of dimensionality in the high-dimensional settings.

\section*{Acknowledgements}
The work J. Huang is supported in part by the U.S. National Science Foundation grant DMS-1916199.
The work of Y. Jiao  is supported in part by the National Science Foundation of China  (No. 11871474)  and by the research fund of KLATASDSMOE of China.


\newpage
\bibliographystyle{apalike}
\bibliography{wgcs_bib.bib}

\begin{thebibliography}{}

\bibitem[Abadi et~al., 2016]{abadi2016tensorflow}
Abadi, M., Barham, P., Chen, J., Chen, Z., Davis, A., Dean, J., Devin, M.,
  Ghemawat, S., Irving, G., Isard, M., et~al. (2016).
\newblock Tensorflow: A system for large-scale machine learning.
\newblock In {\em 12th USENIX Symposium on Operating Systems Design and
  Implementation (OSDI 16)}, pages 265--283.

\bibitem[Arjovsky et~al., 2017]{arjovsky17}
Arjovsky, M., Chintala, S., and Bottou, L. (2017).
\newblock {W}asserstein generative adversarial networks.
\newblock In {\em International Conference on Machine Learning}.

\bibitem[Bauer and Kohler, 2019]{bauer2019deep}
Bauer, B. and Kohler, M. (2019).
\newblock On deep learning as a remedy for the curse of dimensionality in
  nonparametric regression.
\newblock {\em Ann. Statist.}, 47(4):2261--2285.

\bibitem[Bishop and Peres, 2016]{bishop_peres2016}
Bishop, C.~J. and Peres, Y. (2016).
\newblock {\em Fractals in Probability and Analysis}.
\newblock Cambridge Studies in Advanced Mathematics. Cambridge University
  Press.

\bibitem[Bott and Kohler, 2017]{bott2017}
Bott, A.-K. and Kohler, M. (2017).
\newblock Nonparametric estimation of a conditional density.
\newblock {\em Annals of the Institute of Statistical Mathematics},
  69(1):189--214.

\bibitem[Cortez et~al., 2009]{cortez200}
Cortez, P., Cerdeira, A., Almeida, F., Matos, T., and Reis, J. (2009).
\newblock Modeling wine preferences by data mining from physicochemical
  properties.
\newblock {\em Decision Support Systems}, 47(4):547--553.

\bibitem[Dalmasso et~al., 2020]{dalmasso2020conditional}
Dalmasso, N., Pospisil, T., Lee, A.~B., Izbicki, R., Freeman, P.~E., and Malz,
  A.~I. (2020).
\newblock Conditional density estimation tools in python and {R} with
  applications to photometric redshifts and likelihood-free cosmological
  inference.
\newblock {\em Astronomy and Computing}, page 100362.

\bibitem[Dua and Graff, 2017]{Dua:2019}
Dua, D. and Graff, C. (2017).
\newblock {UCI} machine learning repository.
\newblock \url{http://archive.ics.uci.edu/ml}.

\bibitem[Falconer, 2004]{falconer2004fractal}
Falconer, K. (2004).
\newblock {\em Fractal Geometry: Mathematical Foundations and Applications}.
\newblock John Wiley \& Sons.

\bibitem[Fan et~al., 1996]{fan1996}
Fan, J., Yao, Q., and Tong, H. (1996).
\newblock Estimation of conditional densities and sensitivity measures in
  nonlinear dynamical systems.
\newblock {\em Biometrika}, 83(1):189--206.

\bibitem[Goodfellow et~al., 2014]{goodfellow14}
Goodfellow, I., Pouget-Abadie, J., Mirza, M., Xu, B., Warde-Farley, D., Ozair,
  S., Courville, A., and Bengio, Y. (2014).
\newblock Generative adversarial nets.
\newblock In {\em Advances in Neural Information Processing Systems},
  volume~27, pages 2672--2680.

\bibitem[Goodfellow et~al., 2016]{goodfellow2016deep}
Goodfellow, I.~J., Bengio, Y., and Courville, A. (2016).
\newblock {\em Deep Learning}.
\newblock The MIT Press, Cambridge, MA, USA.
\newblock \url{http://www.deeplearningbook.org}.

\bibitem[Gulrajani et~al., 2017]{gulrajani2017improved}
Gulrajani, I., Ahmed, F., Arjovsky, M., Dumoulin, V., and Courville, A. (2017).
\newblock Improved training of wasserstein gans.
\newblock {\em arXiv preprint arXiv:1704.00028}.

\bibitem[Hall et~al., 2004]{hall2004cross}
Hall, P., Racine, J., and Li, Q. (2004).
\newblock Cross-validation and the estimation of conditional probability
  densities.
\newblock {\em Journal of the American Statistical Association},
  99(468):1015--1026.

\bibitem[Hall and Yao, 2005]{hall2005}
Hall, P. and Yao, Q. (2005).
\newblock Approximating conditional distribution functions using dimension
  reduction.
\newblock {\em Annals of Statististics}, 33(3):1404--1421.

\bibitem[Hyndman et~al., 1996]{hyndman1996estimating}
Hyndman, R.~J., Bashtannyk, D.~M., and Grunwald, G.~K. (1996).
\newblock Estimating and visualizing conditional densities.
\newblock {\em Journal of Computational and Graphical Statistics},
  5(4):315--336.

\bibitem[Izbicki and Lee, 2016]{izbicki2016nonparametric}
Izbicki, R. and Lee, A.~B. (2016).
\newblock Nonparametric conditional density estimation in a high-dimensional
  regression setting.
\newblock {\em Journal of Computational and Graphical Statistics},
  25(4):1297--1316.

\bibitem[Izbicki et~al., 2017]{izbicki2017converting}
Izbicki, R., Lee, A.~B., et~al. (2017).
\newblock Converting high-dimensional regression to high-dimensional
  conditional density estimation.
\newblock {\em Electronic Journal of Statistics}, 11(2):2800--2831.

\bibitem[Jiao et~al., 2021]{jiao2021deep}
Jiao, Y., Shen, G., Lin, Y., and Huang, J. (2021).
\newblock Deep nonparametric regression on approximately low-dimensional
  manifolds.
\newblock {\em arXiv 2104.06708}.

\bibitem[Kallenberg, 2002]{kall2002}
Kallenberg, O. (2002).
\newblock {\em Foundations of Modern Probability}.
\newblock Springer-Verlag, New York, 2nd edition.

\bibitem[Kingma and Ba, 2015]{kingma2015adam}
Kingma, D.~P. and Ba, J. (2015).
\newblock Adam: A method for stochastic optimization.
\newblock In {\em Proceedings of the 3rd International Conference on Learning
  Representation}.

\bibitem[Kovachki et~al., 2021]{kovachki2021conditional}
Kovachki, N., Baptista, R., Hosseini, B., and Marzouk, Y. (2021).
\newblock Conditional sampling with monotone gans.
\newblock {\em arXiv 2006.06755}.

\bibitem[LeCun and Cortes, 2010]{mnist2010}
LeCun, Y. and Cortes, C. (2010).
\newblock {MNIST} handwritten digit database.

\bibitem[Liu et~al., 2015]{liu2015faceattributes}
Liu, Z., Luo, P., Wang, X., and Tang, X. (2015).
\newblock Deep learning face attributes in the wild.
\newblock In {\em Proceedings of International Conference on Computer Vision
  (ICCV)}.

\bibitem[Lu et~al., 2020]{lu2020deep}
Lu, J., Shen, Z., Yang, H., and Zhang, S. (2020).
\newblock Deep network approximation for smooth functions.
\newblock {\em arXiv preprint arXiv:2001.03040}.

\bibitem[Lu and Lu, 2020]{lu2020universal}
Lu, Y. and Lu, J. (2020).
\newblock A universal approximation theorem of deep neural networks for
  expressing probability distributions.
\newblock {\em arXiv preprint arXiv:2004.08867}.

\bibitem[Lu et~al., 2017]{lu2017expressive}
Lu, Z., Pu, H., Wang, F., Hu, Z., and Wang, L. (2017).
\newblock The expressive power of neural networks: A view from the width.
\newblock {\em arXiv preprint arXiv:1709.02540}.

\bibitem[Mirza and Osindero, 2014]{mirza2014cgan}
Mirza, M. and Osindero, S. (2014).
\newblock Conditional generative adversarial nets.
\newblock arXiv:1411.1784.

\bibitem[Mohri et~al., 2018]{mohri2018foundations}
Mohri, M., Rostamizadeh, A., and Talwalkar, A. (2018).
\newblock {\em Foundations of Machine Learning}.
\newblock The MIT press.

\bibitem[Montanelli and Du, 2019]{montanelli2019new}
Montanelli, H. and Du, Q. (2019).
\newblock New error bounds for deep relu networks using sparse grids.
\newblock {\em SIAM Journal on Mathematics of Data Science}, 1(1):78--92.

\bibitem[M{\"u}ller, 1997]{muller}
M{\"u}ller, A. (1997).
\newblock Integral probability metrics and their generating classes of
  functions.
\newblock {\em Advances in Applied Probability}, pages 429--443.

\bibitem[Nakada and Imaizumi, 2020]{nakada2020adaptive}
Nakada, R. and Imaizumi, M. (2020).
\newblock Adaptive approximation and generalization of deep neural network with
  intrinsic dimensionality.
\newblock {\em J. Mach. Learn. Res.}, 21:174--1.

\bibitem[Petersen and Voigtlaender, 2018]{petersen2018optimal}
Petersen, P. and Voigtlaender, F. (2018).
\newblock Optimal approximation of piecewise smooth functions using deep relu
  neural networks.
\newblock {\em Neural Networks}, 108:296--330.

\bibitem[Rosenblatt, 1969]{rosenblatt1969}
Rosenblatt, M. (1969).
\newblock Conditional probability density and regression estimators.
\newblock In {\em In Multivariate Analysis II, Ed. P. R. Krishnaiah}, pages
  25--31, New York. Academic Press, New York.

\bibitem[Schmidt-Hieber, 2020]{schmidt2020nonparametric}
Schmidt-Hieber, J. (2020).
\newblock {N}onparametric regression using deep neural networks with {R}e{LU}
  activation function (with discussion).
\newblock {\em Ann. Statist.}, 48(4):1916--1921.

\bibitem[Scott, 1992]{scott1992}
Scott, D.~W. (1992).
\newblock {\em Multivariate Density Estimation: Theory, Practice and
  Visualization}.
\newblock Wiley, New York.

\bibitem[Shen et~al., 2019a]{shen}
Shen, Z., Yang, H., and Zhang, S. (2019a).
\newblock Deep network approximation characterized by number of neurons.
\newblock {\em arXiv preprint arXiv:1906.05497}.

\bibitem[Shen et~al., 2019b]{shen2019nonlinear}
Shen, Z., Yang, H., and Zhang, S. (2019b).
\newblock Nonlinear approximation via compositions.
\newblock {\em Neural Networks}, 119:74--84.

\bibitem[Srebro and Sridharan, 2010]{srebro2010note}
Srebro, N. and Sridharan, K. (2010).
\newblock Note on refined dudley integral covering number bound.
\newblock {\em Unpublished results. http://ttic. uchicago. edu/karthik/dudley.
  pdf}.

\bibitem[Suzuki, 2018]{suzuki2018adaptivity}
Suzuki, T. (2018).
\newblock Adaptivity of deep relu network for learning in besov and mixed
  smooth besov spaces: optimal rate and curse of dimensionality.
\newblock {\em arXiv preprint arXiv:1810.08033}.

\bibitem[Tsybakov, 2008]{npe2008}
Tsybakov, A. (2008).
\newblock {\em Introduction to Nonparametric Estimation}.
\newblock Springer Science \& Business Media.

\bibitem[Van~der Vaart and Wellner, 1996]{vw1996}
Van~der Vaart, A.~W. and Wellner, J.~A. (1996).
\newblock {\em Weak Convergence and Empirical Processes: with Applications to
  Statistics}.
\newblock Springer, New York.

\bibitem[Villani, 2008]{villani2008optimal2}
Villani, C. (2008).
\newblock {\em Optimal Transport: Old and New}.
\newblock Springer Berlin Heidelberg.

\bibitem[Yang et~al., 2021]{yang2021capacity}
Yang, Y., Li, Z., and Wang, Y. (2021).
\newblock On the capacity of deep generative networks for approximating
  distributions.
\newblock {\em arXiv preprint arXiv:2101.12353}.

\bibitem[Yarotsky, 2017]{yarotsky2017error}
Yarotsky, D. (2017).
\newblock Error bounds for approximations with deep relu networks.
\newblock {\em Neural Networks}, 94:103--114.

\bibitem[Yarotsky, 2018]{yarotsky2018optimal}
Yarotsky, D. (2018).
\newblock Optimal approximation of continuous functions by very deep relu
  networks.
\newblock In {\em Conference on Learning Theory}, pages 639--649. PMLR.

\bibitem[Zhou et~al., 2021]{zjlh2021}
Zhou, X., Jiao, Y., Liu, J., and Huang, J. (2021).
\newblock A deep generative approach to conditional sampling.
\newblock {\em arXiv preprint arXiv:2110.10277}.

\end{thebibliography}

\newpage
\setcounter{equation}{0}  
\renewcommand{\theequation}{B.\arabic{equation}}
\setcounter{table}{0}
\renewcommand{\thetable}{A.\arabic{table}}
\setcounter{figure}{0}
\renewcommand{\thefigure}{A.\arabic{figure}}
\setcounter{equation}{0}  

\begin{appendices}
\noindent
\textbf{\LARGE Appendices}

\bigskip\noindent
In the appendices, we include additional numerical results, proofs of the theorems stated in the paper and technical details.

\section{Additional numerical results}
\label{experiments2}

We carry out numerical experiments  to evaluate the finite sample performance of WGCS and illustrate its applications using examples in conditional sample generation,  nonparametric conditional density estimation, visualization of multivariate data, and image generation and reconstruction.
For all the experiments below, the noise random vector $\eta$ is generated from a standard multivariate normal distribution. The discriminator is updated five times per one update of the generator.
We implemented WGCS in
TensorFlow \citep{abadi2016tensorflow} and used the stochastic gradient descent algorithm Adam \citep{kingma2015adam} in training the neural networks.

\subsection{Conditional prediction: the abalone dataset}
The abalone dataset is available at UCI machine learning repository \citep{Dua:2019}.
It contains the number of rings of abalone and other physical measurements. The age of abalone is determined by cutting the shell through the cone, staining it, and counting the number of rings through a microscope, a time-consuming process. Other measurements, which are easier to obtain, are used to predict the number of rings that determines the age.
This dataset contains 9 variables. They are \textit{sex, length, diameter, height, whole weight, shucked weight, viscera weight, shell weight} and \textit{rings}.
Except for the categorical variable \textit{sex}, all the other variables are continuous. The variable \textit{sex} codes three groups: female, male and infant, since the gender of an infant abalone is not known.
We take \textit{rings} as the response $Y \in \mathbb{R}$ and the other measurements as the covariate vector $X \in \mathbb{R}^9$.
The sample size is 4,177. We use 90\% of the data for training and 10\% of the data as the test set.

The generator is a two-layer fully-connected network with 50 and 20 nodes
and the discriminator is also a two-layer network with 50 and 25 nodes, both with ReLU activation function. The noise $\eta\sim N(\mathbf{0}, \mathbf{I}_3).$

\begin{figure}[H]
\centering
\includegraphics[width=5.0 in, height=1.8 in]{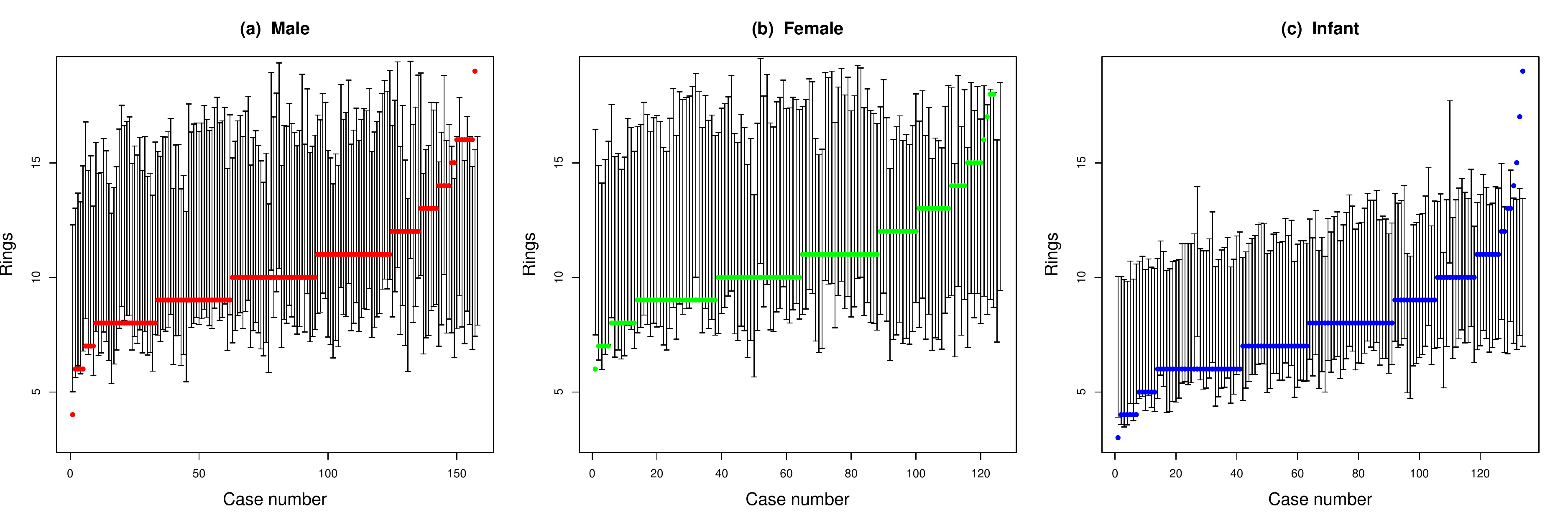}
	\caption{The prediction intervals for the test set.  The 418 abalones in the test set are divided into three groups, (a) male, (b) female,  and (c) infant.}
	\label{fig:abaringsingle}
\end{figure}
To examine the prediction performance of the estimated conditional density, we construct the 90\% prediction interval for the number of rings of each abalone in the testing set. The prediction intervals are shown in Figure \ref{fig:abaringsingle}. The actual number of rings are plotted as a solid dot. The actual coverage for all 418 cases in the testing set is 90.90\%, close
to the nominal level of 90\%. The numbers of rings that are not covered by the prediction intervals are the largest ones in each group.
This dataset was also analyzed in \cite{zjlh2021} using a generative method with the Kullback-Liebler divergence measure. The results are similar.

\subsection{Image generation: MNIST dataset}
We now illustrate the application of  WGCS to high-dimensional data problems and demonstrate that it can easily handle the models when either of both of $X$ and $Y$ are high-dimensional.
We use
the MNIST handwritten digits dataset \citep{mnist2010},  which contains 60,000 images for training and 10,000 images for testing. The images are stored in $28\times28$ matrices with gray color intensity from 0 to 1. Each image is paired with a label in $\{0,1\dots,9\}$.
We use WGCS to perform two tasks: generating images from labels and reconstructing the missing part of an image.

We generate images of handwritten digits given the label. In this problem,  the predictor $X$ is a categorical variable representing the ten digits: $\{0,1,\dots,9\}$ and the response $Y$ represents $28 \times 28$ images.
We use one-hot vectors in $\mathbb{R}^{10}$ to represent these ten categories. So the dimension of $X$ is 10
and the dimension of  $Y$ is $28\times 28=784.$
The response $Y\in [0,1]^{28\times28}$ is a matrix representing the intensity values. For the discriminator $D$, we use a convolutional neural network (CNN) with 3 convolution layers with 128, 256, and 256 filters to extract the features of the image and then concatenate with the label information (repeated 10 times to match the dimension of the features). The concatenated information is sent to a fully connected layer and then to the output layer. For the generator $G$, we concatenate the label information with the random noise vector $\eta \sim N(\mathbf{0}, \mathbf{I}_{100}).$
Then it is fed to a CNN with 3 deconvolution layers with 256, 128, and 1 filters, respectively.

\begin{figure}[H]
	\centering
\includegraphics[width=5.0 in, height=1.8 in]{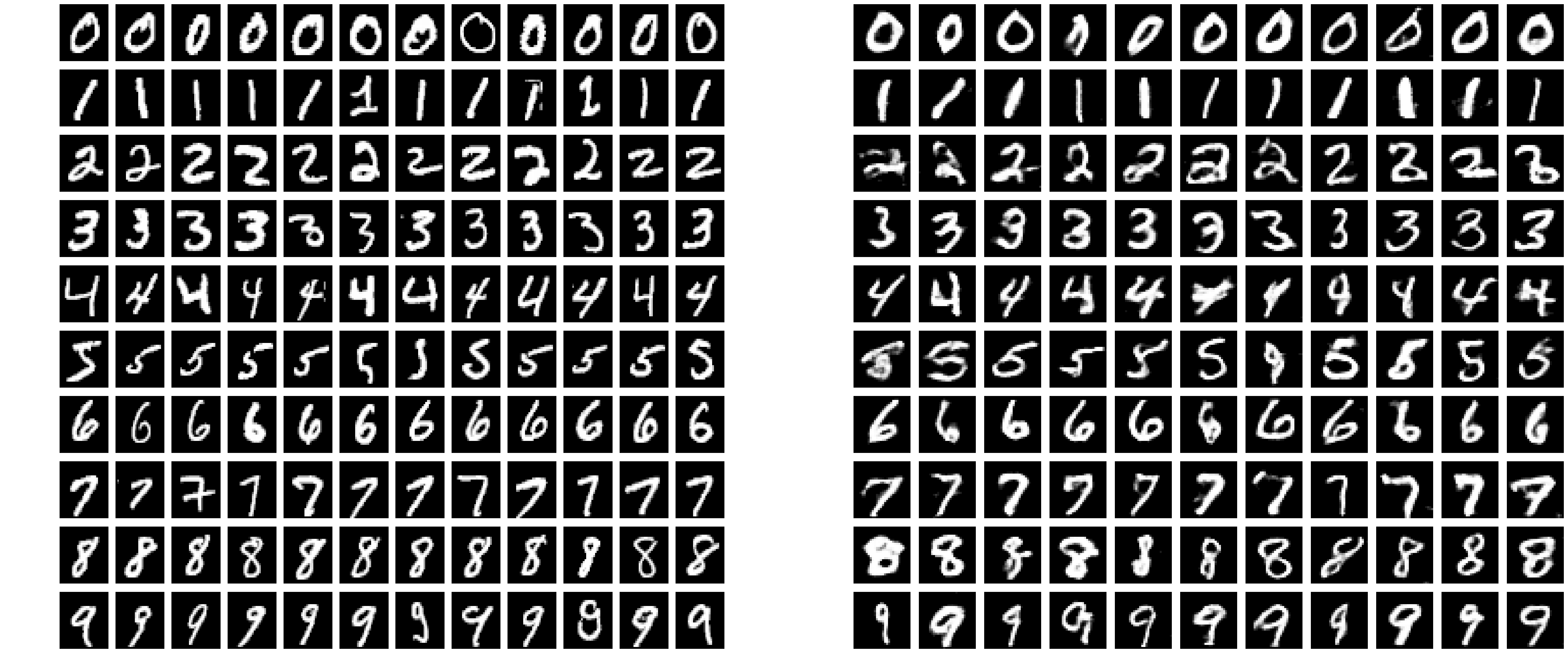}
	\caption{MNIST dataset: real images (left panel) and generated images given the labels (right panel).}
	\label{fig:mnist_xlabel}
\end{figure}

Figure \ref{fig:mnist_xlabel} shows the real images (left panel) and generated images (right panel).
We see that the generated images are similar to the real images and it is hard to distinguish the generated ones from the real images.
Also, there are some differences in the generated images, reflecting the random variations in the generating process.
\subsection{Image generation given labels: CelebA dataset}
We illustrate the application of WGCS to the problem of image generation
with the CelebFaces Attributes (CelebA) dataset \citep{liu2015faceattributes}, which is a large-scale dataset containing more then 200K colored celebrity images with 40 attributes annotation. Here we use 10 features including: \textit{Gender, Young, Bald, Bangs, Blackhair, Blondhair, Eyeglasses, Mustache, Smiling, WearingNecktie}
as binary variables. We  take these features
as $X$. So $X$ is a 
vector consisting of 10 binary components.
We used the aligned and cropped images and further resize them to $96\times96$ as our training set. We take these colored images as the values of $Y$. Therefore, the dimension of $X$ is 10 and the dimension of $Y$ is $96\times96\times3=27,648$.

\begin{table}[H]
	\centering
	\begin{tabular}
{|c|C{0.8cm}C{0.8cm}C{0.9cm}C{0.9cm}C{0.8cm}C{0.7cm}C{0.6cm}C{0.6cm}C{0.9cm}C{0.8cm}|}
		\hline
		Attributes & Gender & Young&Blkhair&Bldhair&Glass&Bald&Mus&Smile& Necktie&Bangs\\
		\hline
		Type 1 &F &Y&N&N&N&N&N&Y&N&N\\
		Type 2 &F&Y&N&Y&N&N&N&Y&N& N\\
		Type 3 &F&Y&Y&N&N&N&N&N&N& N\\
		Type 4 &M&Y&N&N&N&N &N&N&N&N\\
		Type 5 &M&N&N&N&N&N&N&Y &N&N\\
		Type 6 &M&Y&Y&N&N&N&N&Y&N&N\\
		\hline
	\end{tabular}\\

\medskip
	\caption{Attributes for six types of face images in the CelebA dataset:  F$=$Female,  M$=$Male; Y$=$ Yes, N$=$ No.
}
	\label{tab:celeba_type}
\end{table}

The architecture of the discriminator is specified as follows: the 10 dimensional one-hot vector is first expanded to $96\times96\times 10$ by replicating each number to a $96\times 96 $matrix. Then it is concatenated  with image $Y$ on the last axis. Thus the processed data has dimension $96\times96\times 13$. This processed data is then fed into 5 consecutive convolution layers initialized by truncated normal distribution with $SD = 0.01$ with 64,128,256,512,1024 filters respectively. The activation for each convolution layer is a leaky RELU with $\alpha = 0.2$. The strides is 2 and the kernel size is 5. After convolution, it is flattened and connected to a dense layer with one unit as output.

The structure of generator is as follows: the input of $X$ and noise is concatenated and fed to a dense layer with $3\times3\times 1024$ units with RELU activation and then reshaped to $3\times3\times 1024$. Then it goes through 5 layers of deconvolution initialized by truncated normal distribution with $SD = 0.01$ with 512,256,128,64,3 filters respectively. The activation for each intermediate convolution layer is RELU and hyperbolic tangent function for the last layer. The strides is 2 and kernel size is 5.
The optimizer is Adam with $\beta_1 = 0.5$ and learning rate decrease from $0.0001$ to $0.000005$. The noise dimension is 100.

\begin{figure}[H]
	\centering
	\includegraphics[width=5.0 in, height=2.2 in]{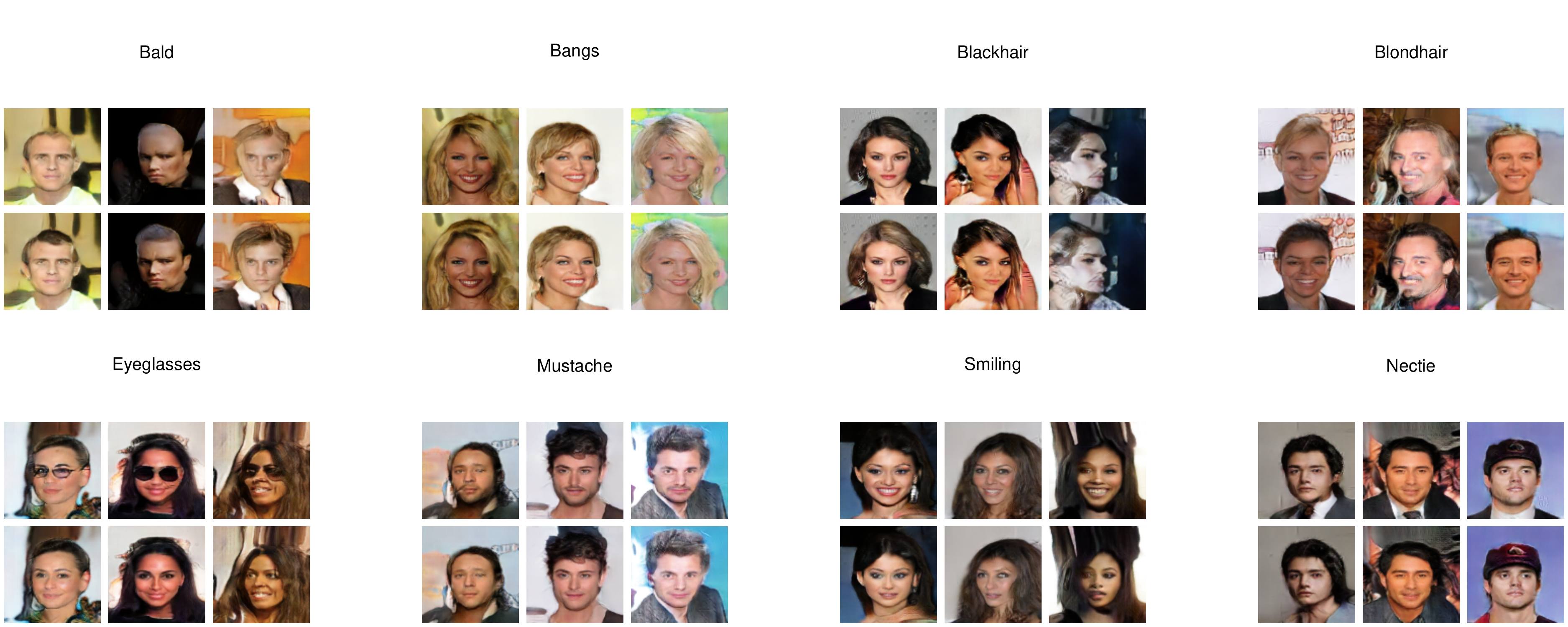}
	\caption{Comparison of the binary features. All images in this figure are generated images. The first row and second row are generated given the same random noise and attributes except for the one noted above.}
	\label{fig:celebafeatures}
\end{figure}

In Figure \ref{fig:celebafeatures}, we present image generation given some specific attributes.
All the images in this figure are generated. For each panel, there are three pairs of faces. Each pair is generated with the same $\eta$ except for the feature labeled above so they look exactly the same except for the specified feature. For example, the first panel shows the generate bald or non-bald faces, the first row includes bald faces and the second row includes non-bald faces.

\section{A high level description of the error analysis}
Below we first present a high level description of the error analysis.
For the estimator $\hG$ of the conditional generator as given in
(\ref{2}), we are interested in bounding the error $d_{\cF^1}(P_{X,\hG},P_{X,Y})$.
Our basic idea is to decompose this error into terms that are easier to analyze.

Let $\{(X'_i,Y'_i), i=1,\ld,n\}$ and $\{\eta'_j, j=1,\ld,n\}$ be ghost samples that are independent of the original samples. Here we introduce ghost samples as a technical tool
for bounding the stochastic error term $\cE_4$ defined below. The details are given in the proof of
Lemma \ref{ep} in the appendix.
We consider $(\hG, \hD)$ based on the empirical version of $\cL(G,D)$ that depends on the original samples $(X_i, Y_i, \eta_i)$ given in (\ref{objE}) and $(\hG', \hD')$ based on the loss function of
the ghost samples $(X'_i, Y'_i, \eta'_i)$,
\begin{align*}
 \cL_n'(G,D)= \fracsumton D(X'_i,G(\eta'_i,X'_i))-\fracsumton D(X'_i,Y'_i).
\end{align*}
Recall the error $d_{\cF^1}(P_{X,\hG},P_{X,Y})$ is defined by
\[
d_{\cF_B^1}(P_{X,\hG},P_{X,Y})=\sup_{f\in \cF_B^1} \{\Ebb f(X, \hG)-\Ebb f(X,Y)\}.
\]
We decompose it as follows:
\begin{align*}
 &d_{\cF_B^1}(P_{X,\hG},P_{X,Y})
    \le \sup_{f\in \cF_B^1}\Big\{ \Ebb f(X,\hG)-\fracsumton f(X'_i, \hat{G}')\Big\}\\
 &   +\sup_{f\in \cF_B^1}
   \Big \{ \fracsumton f(X'_i, \hat{G}')-\fracsumton f(X'_i, Y'_i)\Big\}
    +\sup_{f\in \cF_B^1}\Big\{ \fracsumton f(X'_i, Y'_i)-\Ebb f(X,Y)\Big\} \\
    &:= \cE_4+A + \cE3,
\end{align*}
where  $A=\supcF \{\fracsumton f(X'_i, \hat{G}')-\fracsumton f(X'_i, Y'_i)\}$,
\begin{align}\label{cE3}
\cE_3=\supcF \{\fracsumton f(X'_i, Y'_i)-\Ebb f(X,Y)\},
\end{align}
and
\begin{align}\label{cE4}
\cE_4=\supcF \{\Ebb f(X,\hG)-\fracsumton f(X'_i, \hat{G}')\}.
\end{align}
 By Lemma~\ref{lmaa1}, we have
\begin{align*}
    A &\le 2 \underset{f\in\cF^1}{\sup}\underset{\phi}
{\inf}\|f-D_{\phi}\|_{\infty}+ \sup_{\phi}\Big\{ \fracsumton D_{\phi}(X'_i, \hat{G}')-\fracsumton D_{\phi}(X'_i, Y'_i)\Big\}\\
    &\le 2 \underset{f\in\cF^1}{\sup}\underset{\phi}
{\inf}\|f-D_{\phi}\|_{\infty} + \inf_{\tht}\sup_{\phi}\Big\{ \fracsumton D_{\phi}(X'_i, G_{\tht})-\fracsumton D_{\phi}(X'_i, Y'_i)\Big\}\\
    &\le 2 \cE_1+ \cE_2,
\end{align*}
where
\begin{align}\label{cE1}
\cE_1=\underset{f\in\cF^1}{\sup}\underset{\phi}
{\inf}\|f-D_{\phi}\|_{\infty},
\end{align}
 and
\begin{align}\label{cE2}
\cE_2=\inf_{\tht}\sup_{\phi} \fracsumton \{D_{\phi}(X'_i, G_{\tht})-\fracsumton D_{\phi}(X'_i, Y'_i)\}.
\end{align}

By their definitions, we can see that $\cE_1$ and $\cE_2$ are approximation errors; $\cE_3$ and $\cE_4$ are stochastic errors.

We summarize the above derivation in the following lemma.

\begin{lemma}\label{lma2bound} Let $\hG=G_{\hat{\tht}}$ be the minimax solution in \eqref{2}. Then the bounded Lipschitz distance between $P_{X,\hat{G}}$ and $P_{X,Y}$ can be decomposed as follows.
\begin{align}
    d_{\mathcal{F}_{B}^1}(P_{X,\hat{G}},P_{X,Y})\le 2\mathcal{E}_1+\mathcal{E}_2+\mathcal{E}_3+\mathcal{E}_4.
    \label{decomp}
\end{align}
\end{lemma}

Theorems \ref{thm1} to \ref{thmlowdim} are proved based on the error decomposition
(\ref{decomp})
and by bounding each of the error terms $\mathcal{E}_1$
to $\mathcal{E}_4$.

\section{Supporting Lemmas and Proofs} \label{ap:a}
In this section, we present support lemmas and prove Theorems \ref{thm1} to \ref{thmlowdim}.
We first prove the following lemma.
\begin{lemma}\label{lmaa1}
For any symmetric function classes $\mathcal{F}$ and $\mathcal{H}$, denote the approximation error $\mathcal{E}(\mathcal{H},\mathcal{F})$ as
\begin{align*}
    \mathcal{E}(\mathcal{H},\mathcal{F}):=
    \underset{h\in\mathcal{H}}{\sup}\underset{f\in\mathcal{F}}{\inf}||h-f||_{\infty},
\end{align*}
then for any probability distributions $\mu$ and $\nu$,
\begin{align*}
    d_{\mathcal{H}}(\mu,\nu)-d_{\mathcal{F}}(\mu,\nu)\leq 2\mathcal{E}(\mathcal{H},\mathcal{F}).
\end{align*}
This inequality can be extended to an empirical version by using empirical measures.
\end{lemma}

\begin{proof}
By the definition of supremum, for any $\epsilon>0$, there exists $h_{\epsilon}\in\mathcal{H}$ such that
\begin{align*}
    d_{\mathcal{H}}(\mu,\nu):&=\underset{h\in\mathcal{H}}{\sup}[\mathbb{E}_{\mu}h-\mathbb{E}_{\nu}h]\\
    &\leq \mathbb{E}_{\mu}h_{\epsilon}-\mathbb{E}_{\nu}h_{\epsilon}+\epsilon\\
    &=\underset{f\in\mathcal{F}}{\inf}[\mathbb{E}_{\mu}(h_{\epsilon}-f)-\mathbb{E}_{\nu}(h_{\epsilon}-f)+\mathbb{E}_{\mu}(f)-\mathbb{E}_{\nu}(f)]+\epsilon\\
    &\leq 2\underset{f\in \mathcal{F}}{\inf}||h_{\epsilon}-f||_{\infty}+d_{\mathcal{F}}(\mu,\nu)+\epsilon\\
    &\leq 2\mathcal{E}(\mathcal{H},\mathcal{F})+d_{\mathcal{F}}(\mu,\nu)+\epsilon,
\end{align*}
where the last line is due to the definition of $\mathcal{E}(\mathcal{H},\mathcal{F})$.
\end{proof}

\subsection{An equivalent statement}
We hope that functions in the evaluation class $\cF^1$ are defined on a bounded domain so we can apply existing neural nets approximation theorems to bound the approximation error $\cE_1$.
It motivates us to first show that proving the desired convergence rate is equivalent to establishing the same convergence rate but with the domain restricted function class $\cF^1_n:=\{f|_{B_{\infty}(2\log n)}:f\in\cF^1\}$ as the evaluation class under Assumption \ref{asp1}.

Suppose Assumption \ref{asp1} holds. By the Markov inequality we have
\begin{align}
    P(\|Z\|>\log n)\leq \frac{\mathbb{E} \|Z\| \mathbbm{1}_{\{\|Z\|>\log n\}}}{\log n}= O( n^{-\frac{(\log n)^{\delta}}{d+q}}/\log n),
    \label{eq:2.3}
\end{align}
where $Z= (X,Y)$. The bounded Lipschitz distance is defined as
\begin{align}
d_{\cF^1}(P_{X,Y}, P_{X,\hG})&=\supcF \Ebb f(X,Y)-\Ebb f(X, \hG).
\label{eq:2.4}
\end{align}
The first term above can be decomposed as
\begin{align*}
  \mathbb{E} f(Z)&= \mathbb{E} f(Z)\mathbbm{1}_{\|Z\|\le\log n}+\mathbb{E} f(Z)\mathbbm{1}_{\|Z\|>\log n}.
\end{align*}
For any $f\in \mathcal{F}^1$ and a fixed point $\|z_0\|<\log n$, due to the Lipschitzness of $f$, the second term above satisfies
\begin{align*}
    |\mathbb{E} f(Z)\mathbbm{1}_{\|Z\|>\log n}|\le &|\mathbb{E} f(Z)\mathbbm{1}_{\|Z\|>\log n}-\mathbb{E} f(z_0)\mathbbm{1}_{\|Z\|>\log n}|+|\mathbb{E} f(z_0)\mathbbm{1}_{\|Z\|>\log n}|\\
    \le & \mathbb{E}\|Z-z_0 \|\mathbbm{1}_{\|Z\|>\log n}+B  P(\|Z\|>\log n)\\
    = & O(n^{-\frac{(\log n)^{\delta}}{d+q}}),
\end{align*}
where the second inequality is due to lipschitzness and boundedness of $f$, and the last inequality is due to Assumption \ref{asp1} and \eqref{eq:2.3}. The second term in \eqref{eq:2.4} can be dealt similarly due to Condition \ref{cond1} for the network $G_{\theta}$. Hence, restricting the evaluation class to $\cF_n^1$ will not affect the convergence rate in the main results, i.e. $O(n^{-\frac{1}{d+q}})$. Due to this fact, to keep notation simple,
we denote $\mathcal{F}_n^1$ as $\mathcal{F}^1$ in the following sections.

\subsection{Bounding the errors}
We first state two lemmas for controlling the approximation error $\cE_2$ in (\ref{cE2}).
Let $\mathcal{S}^d (z_0,\ldots,z_{N+1})$ be the set of all continuous piecewise linear functions $f:\Rbb\mapsto \Rbb^d,$ which have breakpoints only at $z_0<z_1<\cdots<z_N<z_{N+1}$ and are constant on $(-\infty,z_0)$ and $(z_{N+1},\infty)$.
The following lemma is from \cite{yang2021capacity}.

\begin{lemma}\label{appdis}
Suppose that $W\ge 7d+1$, $L\ge2$ and $N\le (W-d-1)\floor*{\frac{W-d-1}{6d}}\floor*{\frac{L}{2}}$. Then for any $z_0<z_1<\cdots<z_N<z_{N+1}$, $\mathcal{S}^d (z_0,\ldots,z_{N+1})$ can be represented by a ReLU FNNs with width and depth no larger than $W$ and $L$, respectively.
\end{lemma}

 If we choose $N=(W-d-1)\floor*{\frac{W-d-1}{6d}}\floor*{\frac{L}{2}}$, a simple calculation shows $c W^2 L/d\le N\le C W^2 L/d$ with $c=1/384$ and $C=1/12$. This means when the number of breakpoints are moderate compared with the network structure, such piecewise linear functions are expressible by feedforward ReLU networks. The next result shows that we can apply piecewise linear functions to push each $\eta'_i$ to $Y'_i$.

\begin{lemma}\label{e2}
Suppose probability measure $\nu$ supported on $\mathbb{R}$ is absolutely continuous w.r.t. Lebesgue measure, and probability meansure $\mu$ is supported on $\mathbb{R}^q$. $\eta_i$ and $y_i$ are i.i.d. samples from $\nu$ and $\mu$, respectively for $i\in[n]$. Then there exist generator ReLU FNN $g: \mathbb{R}\mapsto\mathbb{R}^q$ maps $\eta_i$ to $y_i$ for all $i$. Moreover, such $g$ can be obtained by properly specifying $W_2^2 L_2= c_2 q n $ for some constant $12\le c_2\le 384$. 
\end{lemma}

\begin{proof}
By the absolute continuity of $\nu$, all the $\eta_i$ are distinct a.s.. Without loss of generality, assume $\eta_1<\eta_2<\ldots<\eta_n$ (or we can reorder them). Let $\eta_{i+1/2}$ be any point between $\eta_i$ and $\eta_{i+1}$ for $i\in\{1,2,\ldots,n-1\}$. We define the continuous piece-wise linear function $g: \mathbb{R}\mapsto\mathbb{R}^q$ by
\begin{align*}
    g(\eta)=
    \begin{cases}
    y_1 & \eta\leq \eta_{1},\\
    \frac{\eta-\eta_{i+1/2}}{\eta_i-\eta_{i+1/2}}y_i+\frac{\eta-\eta_{i}}{\eta_{i+\frac{1}{2}}-\eta_{i}}y_{i+1} & \eta=(\eta_i, \eta_{i+1/2}), \text{ for } i=1,\ldots,n-1,\\
    y_{i+1} & \eta\in[\eta_{i+1/2},\eta_{i+1}], \text{ for } i=1,\ldots,n-2,\\
    y_n & \eta\geq \eta_{n-1+1/2}.
    \end{cases}
\end{align*}
Then $g$ maps $\eta_i$ to $y_i$ for all $i$. By Lemma \ref{appdis}, $g\in \mathcal{NN}(W_2,L_2)$ if $n\le (W_2-q-1)\floor*{\frac{W_2-q-1}{6q}}\floor*{\frac{L_2}{2}}$. Taking $n= (W_2-q-1)\floor*{\frac{W_2-q-1}{6q}}\floor*{\frac{L_2}{2}}$, a simple calculation shows $W^2_2 L_2 = cqn$ for some constant $12\le c\le 384$.
\end{proof}

\medskip\noindent
\begin{proof}[(Convergence rate of $\Ebb\cE_4$)]
By Lemma \ref{lma9}, we have
\begin{align*}
    \log \mathcal{N}(\epsilon,\mathcal{F}^1,\|\cdot\|_{\infty})\precsim \left(\frac{\log n}{\epsilon}\right)^{d+q}.
\end{align*}
Taking $\delta=n^{-\frac{1}{d+q}}\log n$ and applying Lemma \ref{ep}, we obtain
\begin{align*}
    \mathbb{E}\mathcal{E}_4&=O\left(\delta+n^{-\frac{1}{2}}(\log n)^{\frac{d+q}{2}}\int_{\delta}^M \epsilon^{-\frac{d+q}{2}} d\epsilon \right)\\
    &= O\left(\delta+n^{-\frac{1}{2}}(\log n)^{\frac{d+q}{2}}\delta^{1-\frac{d+q}{2}}\right)\\
    &= O(1) (d+q)^{1/2}\left(n^{-\frac{1}{d+q}}\log n+n^{-\frac{1}{d+q}}\log n \right)\\
    &= O(1) (d+q)^{1/2}\left(n^{-\frac{1}{d+q}}\log n\right).
\end{align*}
\end{proof}

We now consider a way to characterize the distance between  two joint distributions of $(X,\hG(\eta,X))$ and $(X,Y)$ by the Integral Probability Metric (IPM,  \cite{muller})
with respect to the uniformly bounded 1-Lipschitz function class $\mathcal{F}^1$,
which is defined as
\begin{align*}
\mathcal{F}^1:=\{f: \mathbb{R}^{d+q}\mapsto \mathbb{R} \mid |f(z_1)-f(z_2)|\leq \|z_1-z_2\| \text{ and }\|f\|_{\infty}\leq B \text{ for some } B>0\}.
\end{align*}
The metric $d_{\cF^1}$ is also known as the bounded Lipschitz metric ($d_{BL}$) which metricizes the weak topology on the space of probability distributions. Let $P_{X \hG}$ be the joint distribution of $(X,\hG(\eta,X))$ and $P_{X,Y}$ be the joint distribution of $(X,Y)$. We are bounding the error $\Ebb_{\hG}d_{\cF^1}(P_{X,\hG},P_{X,Y})$.

We consider the error bound under $d_{BL}$ because the approximation errors will be unbounded if we do not impose the uniform bound for the evaluation class $\cF$. However, if
$P_{X,Y}$ has bounded support, then $d_{BL}$ will be essentially $d_{W_1}$.

For convenience, we restate Lemma \ref{lma2bound} below.

\begin{lemma}\label{lma2bounda}
Let $\hG=G_{\hat{\tht}}$ be the minimax solution in \eqref{2}. Then the bounded Lipschitz distance between $P_{X,\hat{G}}$ and $P_{X,Y}$ can be decomposed as follows.
\begin{align*}
    d_{\mathcal{F}^1}(P_{X,\hat{G}},P_{X,Y})\le 2\mathcal{E}_1+\mathcal{E}_2+\mathcal{E}_3+\mathcal{E}_4.
\end{align*}
\end{lemma}

\begin{remark}
As far as we know, all the tools for controlling approximation error $\cE_1$ require that the approximated functions be defined on a bounded domain. Hence, the restriction on $2\log n$ cube for $\cF^1$ is technically necessary for controlling $\cE_1$. The fact that we used uniform bound $B$ in the proof also explains that we can only consider $d_{BL}$ for the unbounded support case. However, in the bounded support case, this statement is obviously unnecessary and we are able to obtain results under 1-Wasserstein distance.
\end{remark}

\subsection{Bounding the error terms}

\noindent
\textbf{Bounding $\cE_1$.}
The discriminator approximation error
$\cE_1=\sup_{f\in\cF^1}\inf_{\phi}\|f-D_{\phi}\|_{\infty}$ describes how well the discriminator neural network class is in the task of approximating functions from the Lipschitz class $\mathcal{F}^1$. Intuitively speaking, larger discriminator class architecture will lead to smaller $\mathcal{E}_1$. There has been much recent work on the approximation power of deep neural networks \citep{
lu2020deep, lu2017expressive,
montanelli2019new,
petersen2018optimal,
 shen, shen2019nonlinear, suzuki2018adaptivity,
 yarotsky2017error, yarotsky2018optimal}, quantitatively or non-quantitatively, asymptotically or non-asymptotically.
The next lemma is a quantitative and non-asymptotic result from \cite{shen}.

\begin{lemma}[\cite{shen} Theorem 4.3]\label{lma2shen}
Let $f$ be a Lipschitz continuous function defined on $B^{d+q}_{\infty}(R)$. For arbitrary $W_1, L_1\in \mathbb{N}_+$, there exists a function $D_{\phi}$ implemented by a ReLU feedforward neural network with width no more than $W_1$ and depth no more than $L_1$ such that
\begin{align*}
    ||f-D_{\phi}||_{\infty}\precsim   R\sqrt{d+q}(W_1 L_1)^{-\frac{2}{d+q}}.
\end{align*}
\end{lemma}
When balancing the errors, we can let the discriminator structure be $W_1 L_1\ge \sqrt{n}$ and $R=2\log n$ so that $\cE_1$ is of the order $\sqrt{d+q}n^{-\frac{1}{d+q}}\log n$, which is the same order of the statistical errors.

\medskip\noindent
\textbf{Bounding $\cE_2$.}
The generator approximation error $\cE_2=\inf_{\tht}\sup_{\phi} n^{-1}\sum_{i=1}^n D_{\phi}(X'_i, G_{\tht})- n^{-1}\sum_{i=1}^n D_{\phi}(X'_i, Y'_i)$ describes how powerful the generator class is to realize the empirical version of noise outsourcing lemma. If we can find a ReLU FNN $G_{\tht_0}$ such that $G_{\tht_0}(\eta'_i,X'_i)=Y'_i$ for all $i\in[n]$, then $\cE_2=0$.
To ease the problem, it suffices to find a $\Rbb^m\mapsto\Rbb^q$ ReLU FNN that maps all the $\eta'_i$ to $Y'_i$. If such a ReLU network exists, then the desired $G_{\tht_0}$ exists automatically by ignoring the input $X'_i$'s. The existence of such a neural network is guaranteed by the Lemma \ref{e2} in appendix, where the structure of the generator network is to be set as $W_2^2 L_2=cqn$ for some constant $12<c<384$. Hence by properly setting the generator network in this way, we can realize $\cE_2=0$. Note that Lemma \ref{e2} holds under the condition that the range of $G_{\theta}$ covers the support of $P_{Y}$. Since we imposed Condition \ref{cond1}, this is not always satisfied. However, assumption \ref{asp1} controls the probability of the bad set where $\cE_2\ne 0$ and we can show that the desired convergence rate is not affected by the bad set.

\medskip\noindent
\textbf{Bounding $\cE_3$.}
The statistical error $\cE_3=\supcF  n^{-1}\sum_{i=1}^n f(X'_i, Y'_i)-\Ebb f(X,Y)$ quantifies how close the empirical distribution and the true target are under bounded Lipschitz distance. The following lemma is a standard result quantifying such a distance.

\begin{lemma}[\cite{lu2020universal} Proposition 3.1]
\label{lu}
Assume that probability distribution $\pi$ on $\Rbb^d$ satisfies that $M_3=\Ebb _{\pi} |X|^3<\infty$, and let $\hat{\pi}_n$ be its empirical distribution. Then
\[
\Ebb d_{W_1}(\hat{\pi}_n,\pi)\precsim \sqrt{d}n^{-\frac{1}{d}}.
\]
\end{lemma}
The finite moment condition is satisfied due to \eqref{eq:2.3} and $E|X|^3=\int_0^{\infty}3t^2 P(|X|>t)dt$. Recall that $d_{\cF^1}(\hat{\pi}_n,\pi)\le d_{W_1}(\hat{\pi}_n,\pi)$, hence we have
\begin{align}
\label{E3b}
\Ebb \cE_3\precsim \sqrt{d+q}n^{-\frac{1}{d+q}}.
\end{align}

\medskip\noindent
\textbf{Bounding $\cE_4$.}
Similar to $\cE_3$, the statistical error $\cE_4=\supcF \Ebb f(X,\hG)- n^{-1}\sum_{i=1}^n f(X'_i, \hat{G}')$ describes the distance between the distribution of $(X,\hG)$ and its empirical distribution. We need to introduce the empirical Rademacher complexity $\hat{\mathcal{R}}_n(\mathcal{F})$ to quantify it.
Define the empirical Rademacher complexity of function class $\mathcal{F}$ as
\begin{align*}
    \hat{\mathcal{R}}_n(\mathcal{F}):=\mathbb{E}_{\epsilon}\underset{f\in\mathcal{F}}{\sup}\frac{1}{n}\sum_{i=1}^n\epsilon_i f(X_i, \hG),
\end{align*}
where $\epsilon=(\epsilon_1,\ldots,\epsilon_n)$ are i.i.d. Rademacher variables, i.e. uniform $\{-1,1\}$.

The next lemma shows how to bound $\mathbb{E}\mathcal{E}_4$ by empirical Rademacher complexity using symmetrization and chaining argument. We include the proof here
for completeness, which can also be found in \cite{srebro2010note}.

\begin{lemma}[Random Covering Entropy Integral]\label{ep}
Assume $\sup_{f\in\mathcal{F}}||f||_{\infty}\leq B$. For any distribution $\mu$ and its empirical distribution $\hat{\mu}_n$, we have
\begin{align}
\label{ep1}
    \mathbb{E}[d_{\mathcal{F}}(\hat{\mu}_n,\mu)]\leq2\mathbb{E} \hat{\mathcal{R}}_n(\mathcal{F})\leq \mathbb{E}\inf_{0<\delta<B}\left(4\delta+\frac{12}{\sqrt{n}}\int_{\delta}^{B}\sqrt{\log \mathcal{N}(\epsilon,\mathcal{F},L_{\infty}(P_n))}d\epsilon\right).
\end{align}
\end{lemma}

\begin{proof}[{of Lemma \ref{ep}}]
The first inequality in (\ref{ep1})  can be proved by symmetrization. We have
\begin{align*}
\mathbb{E}\cE_4&=\mathbb{E}\underset{f\in\mathcal{F}^1}{\sup}\Ebb f(X,\hG)-\fracsumton f(X'_i, \hat{G}')\\
&=\mathbb{E}\underset{f\in\mathcal{F}^1}{\sup}[\Ebb \frac{1}{n}\sum_{i=1}^n f(X_i,\hat{G})-\frac{1}{n}\sum_{i=1}^n f(X'_i, \hat{G}')]\\
&\leq \mathbb{E}\underset{f\in\mathcal{F}^1}{\sup}[\frac{1}{n}\sum_{i=1}^n f(X_i,\hat{G})-\frac{1}{n}\sum_{i=1}^n f(X'_i, \hat{G}')]\\
&=\mathbb{E}\underset{f\in\mathcal{F}^1}{\sup}[\frac{1}{n}\sum_{i=1}^n \epsilon_i(f(X_i,\hat{G})-f(X'_i, \hat{G}'))]\\
&\leq 2\mathbb{E} \hat{\mathcal{R}}_n(\mathcal{F}^1),
\end{align*}
where the first inequality is due to Jensen inequality, and the third equality is because that $f(X_i,\hat{G})-f(X'_i, \hat{G}')$ has symmetric distribution, which is the reason why we introduce the ghost sample.

The second inequality in (\ref{ep1})  can be proved by chaining. Let $\alpha_0=B$ and for any $j\in \mathbb{N}_+$ let $\alpha_j=2^{-j}B$. For each $j$, let $T_i$ be a $\alpha_i$-cover of $\mathcal{F}^1$ w.r.t. $L_2(P_n)$ such that $|T_i|=\mathcal{N}(\alpha_i,\mathcal{F}^1,L_2(P_n))$. For each $f\in\mathcal{F}^1$ and $j$, pick a function $\hat{f}_i\in T_i$ such that $||\hat{f}_i-f||_{L_2(P_n)}<\alpha_i$. Let $\hat{f}_0=0$ and for any $N$, we can express $f$ by chaining as
\begin{align*}
    f=f-\hat{f}_N+\sum_{i=1}^N(\hat{f}_i-\hat{f}_{i-1}).
\end{align*}
Denote $O_i:=(X_i,\hat{G})$. Hence for any $N$, we can express the empirical Rademacher complexity as
\begin{align*}
    \hat{\mathcal{R}}_n(\mathcal{F}^1)&=\frac{1}{n}\mathbb{E}_{\epsilon}\sup_{f\in\mathcal{F}^1}\sum_{i=1}^n \epsilon_i\left(f(O_i)-\hat{f}_N(O_i)+\sum_{j=1}^N(\hat{f}_j(O_i)-\hat{f}_{j-1}(O_i))\right)\\
    & \leq \frac{1}{n}\mathbb{E}_{\epsilon}\sup_{f\in\mathcal{F}^1}\sum_{i=1}^n \epsilon_i\left(f(O_i)-\hat{f}_N(O_i)\right) +\sum_{i=1}^n\frac{1}{n}\mathbb{E}_{\epsilon}\sup_{f\in\mathcal{F}^1} \sum_{j=1}^N\epsilon_i\left(\hat{f}_j(O_i)-\hat{f}_{j-1}(O_i)\right)\\
    & \leq ||\epsilon||_{L_2(P_n)} \sup_{f\in\mathcal{F}^1}||f-\hat{f}_N||_{L_2(P_n)}+\sum_{i=1}^n\frac{1}{n}\mathbb{E}_{\epsilon}\sup_{f\in\mathcal{F}^1} \sum_{j=1}^N\epsilon_i\left(\hat{f}_j(O_i)-\hat{f}_{j-1}(O_i)\right)\\
    &\leq \alpha_N+\sum_{i=1}^n\frac{1}{n}\mathbb{E}_{\epsilon}\sup_{f\in\mathcal{F}^1} \sum_{j=1}^N\epsilon_i\left(\hat{f}_j(O_i)-\hat{f}_{j-1}(O_i)\right),
\end{align*}
where $\epsilon=(\epsilon_1,\ldots,\epsilon_n)$ and the second-to-last inequality is due to Cauchy–Schwarz. Now the second term is the summation of empirical Rademacher complexity w.r.t. the function classes $\{f'-f'': f'\in T_j, f''\in T_{j-1}\}$, $j=1,\ldots, N$. Note that
\begin{align*}
    ||\hat{f}_j-\hat{f}_{j-1}||^2_{L_2(P_n)}&\leq \left(||\hat{f}_j-f||_{L_2(P_n)}+||f-\hat{f}_{j-1}||_{L_2(P_n)}\right)^2\\
    &\leq (\alpha_j+\alpha_{j-1})^2\\
    &=3 \alpha_j^2.
\end{align*}
Massart's lemma (Lemma \ref{mas} below) states that if for any finite function class $\mathcal{F}$, $\sup_{f\in\mathcal{F}}||f||_{L_2(P_n)}\leq B $, then we have
\begin{align*}
\hat{\mathcal{R}}_n(\mathcal{F})\leq\sqrt{\frac{2 B^2\log (|\mathcal{F}|)}{n}}.
\end{align*}
Applying Massart's lemma to the function classes $\{f'-f'': f'\in T_j, f''\in T_{j-1}\}$, $j=1,\ldots, N$, we get that for any $N$,
\begin{align*}
    \hat{\mathcal{R}}_n(\mathcal{F}^1)&\leq \alpha_N+\sum_{j=1}^N 3\alpha_j\sqrt{\frac{2\log (|T_j|\cdot|T_{j-1}|)}{n}}\\
    &\leq \alpha_N+6\sum_{j=1}^N\alpha_j\sqrt{\frac{\log (|T_j|)}{n}}\\
    &\leq \alpha_N+12\sum_{j=1}^N(\alpha_j-\alpha_{j+1})\sqrt{\frac{\log \mathcal{N}(\alpha_j,\mathcal{F}^1, L_2(P_n))}{n}}\\
    &\leq \alpha_N+12\int_{\alpha_{N+1}}^{\alpha_0}\sqrt{\frac{\log \mathcal{N}(r,\mathcal{F}^1, L_2(P_n))}{n}}dr,
\end{align*}
where the third inequality is due to $2(\alpha_j-\alpha_{j+1})=\alpha_j$. Now for any small $\delta>0$ we can choose $N$ such that $\alpha_{N+1}\leq \delta <\alpha_N$. Hence,
\begin{align*}
    \hat{\mathcal{R}}_n(\mathcal{F}^1)&\leq 2\delta +12\int_{\delta/2}^{B}\sqrt{\frac{\log \mathcal{N}(r,\mathcal{F}^1, L_2(P_n))}{n}}dr.
\end{align*}
Since $\delta>0$ is arbitrary, we can take $\inf$ w.r.t. $\delta$ to get
\begin{align*}
    \hat{\mathcal{R}}_n(\mathcal{F}^1)&\leq\inf_{0<\delta<B}\left( 4\delta +12\int_{\delta}^{B}\sqrt{\frac{\log \mathcal{N}(r,\mathcal{F}^1, L_2(P_n))}{n}}dr\right).
\end{align*}
Now the result follows due to the fact that for any function class $\cF$ and samples,
\begin{align*}
    \mathcal{N}(\epsilon,\mathcal{F}, L_2(P_n))\le \mathcal{N}(\epsilon,\mathcal{F},L_{\infty}(P_n))\le \mathcal{N}(\epsilon,\mathcal{F},\|\cdot\|_{\infty}).
\end{align*}
This completes the proof of Lemma \ref{ep}.
\end{proof}

In $\cE_4=\supcF \Ebb f(X,\hG)-\fracsumton f(X'_i, \hat{G}')$, we used the discriminator network $\hG'$ obtained from the ghost samples for the empirical distribution. The reason is that symmetrization requires two distributions being the same.
In our settings, $(X_i, \hG(X_i,\eta_i))$ and $(X'_i, \hG(X'_i,\eta'_i))$ do not have the same distribution, but $(X_i, \hG(X_i,\eta_i))$ and $(X'_i, \hG'(X'_i,\eta'_i))$ do.  Recall that we have restricted $\cF^1$ to $B_{\infty}(2\log n)$. Since $\mathcal{N}(\epsilon,\mathcal{F},L_{\infty}(P_n))\le \mathcal{N}(\epsilon,\mathcal{F},\|\cdot\|_{\infty})$, now it suffices to bound the covering number $\mathcal{N}(\epsilon,\mathcal{F}^1|_{B_{\infty}(2\log n)},\|\cdot\|_{\infty})$.

Lemma \ref{lma9} below provides an upper bound for the covering number of Lipschitz class. It is a direct corollary of \cite[Theorem 2.7.1]{vw1996}.

\begin{lemma}\label{lma9}
 Let $\mathcal{X}$ be a bounded, convex subset of $\mathbb{R}^d$ with nonempty interior. There exists a constant $c_d$ depending only on $d$ such that
 \begin{align*}
     \log \mathcal{N}(\epsilon, \mathcal{F}^1(\mathcal{X}),||\cdot||_{\infty})\leq c_d \lambda(\mathcal{X}^1)\left(\frac{1}{\epsilon}\right)^{d}
 \end{align*}
for every $\epsilon>0$, where $\mathcal{F}^1(\mathcal{X})$ is the 1-Lipschitz function class defined on $\mathcal{X}$, and $\lambda(\mathcal{X}^1)$ is the Lebesgue measure of the set $\{x:||x-\mathcal{X}||<1\}$.
\end{lemma}

Applying Lemmas \ref{ep} and \ref{lma9}
and taking $\delta=C\sqrt{d+q}n^{-\frac{1}{d+q}}\log n$ for some constant $C>0$, we have
\begin{align}
\label{E4b}
    \mathbb{E}\mathcal{E}_4&= O\left(  \sqrt{d+q}n^{-\frac{1}{d+q}}\log n\right).
\end{align}

\subsection{Proofs of the theorems}

We are now ready to prove Theorems \ref{thm1} and \ref{thm2}.

\begin{proof}[(\textbf{of Theorem \ref{thm1}})]

By taking $W_1 L_1=\ceil*{\sqrt{n}}$ and $R=2\log n$ in Lemma \ref{lma2bound}, we get $\cE_1\precsim n^{-\frac{1}{d+q}}\log n$. Lemma \ref{e2} states that $\cE_2=0$ as long as the range of $G_{\tht}$ covers all the $Y'_i$, i.e. $\max_{1\le i\le n} \|Y'_i\|_{\infty}\le \log n$. Hence the nice set $H:= \{\max_{1\le i\le n}\|Z'_i\|\le \log n\}$ is where $\cE_2=0$, and $P(H^c)=1-P(H)\le 1- (1-C \frac{n^{-\frac{(\log n)^{\delta}}{d+q}}}{\log n})^n \le C n^{-\frac{(\log n)^{\delta}}{d+q}}/\log n$. Also, we have $\Ebb \cE_3\precsim n^{-\frac{1}{d+q}}$ and $\Ebb \cE_4\precsim n^{-\frac{1}{d+q}}\log n$ by (\ref{E3b}) and (\ref{E4b}), respectively.
Therefore, by Lemma \ref{decomp}, we have
\begin{align*}
    \Ebb d_{\mathcal{F}^1}(P_{X,\hat{G}},P_{X,Y})&\le  \Ebb d_{\mathcal{F}^1}(P_{X,\hat{G}},P_{X,Y})\mathbbm{1}_{H}+\Ebb d_{\mathcal{F}^1}(P_{X,\hat{G}},P_{X,Y})\mathbbm{1}_{H^c}\\
    &\le (2\mathcal{E}_1+\mathcal{E}_2 +\Ebb\mathcal{E}_3+\Ebb \mathcal{E}_4)\mathbbm{1}_{H}+ 2 B P(H^c)\\
    &\precsim n^{-\frac{1}{d+q}}\log n +0 +n^{-\frac{1}{d+q}}+n^{-\frac{1}{d+q}}\log n + n^{-\frac{(\log n)^{\delta}}{d+q}}/\log n\\
    & \precsim  n^{-\frac{1}{d+q}}\log n.
\end{align*}
This completes the proof of Theorem \ref{thm1}.
\end{proof}

\begin{proof}[\textbf{of Theorem \ref{thm2}}]

By taking $W_1 L_1=\ceil*{\sqrt{n}}$ and $R=M$ in Lemma \ref{lma2shen}, we get $\cE_1\precsim n^{-\frac{1}{d+q}}$. Since the range of $G_{\tht}$ covers all the $Y'_i$, we have $\cE_2=0$. Also, we have $\Ebb \cE_3\precsim n^{-\frac{1}{d+q}}$ by previous results. Similar to the procedure for obtaining the convergence rate of $\Ebb\cE_4$, we get $\Ebb \cE_4\precsim n^{-\frac{1}{d+q}}$. In all by Lemma \ref{decomp}, we have
\begin{align*}
    \Ebb d_{\mathcal{F}^1}(P_{X,\hat{G}},P_{X,Y})&\le 2\mathcal{E}_1+\mathcal{E}_2 +\Ebb\mathcal{E}_3+\Ebb \mathcal{E}_4\\
    &\precsim n^{-\frac{1}{d+q}} +0 +n^{-\frac{1}{d+q}}+n^{-\frac{1}{d+q}}\\
    & \precsim  n^{-\frac{1}{d+q}}.
\end{align*}
This completes the proof of Theorem \ref{thm2}.
\end{proof}


We now prove Corollary \ref{thm3}.

\begin{proof}[\textbf{of Corollary \ref{thm3}}]
It suffices to show $\Ebb_{X} d_{W_1}(P_{\hat{G}(\eta,X)}, P_{Y|X})\le d_{W_1}(P_{X,\hat{G}}, P_{X,Y})$. By the definition of Wasserstein distance, we have $d_{W_1}(P_{\hat{G}(\eta,X=x)}, P_{Y|X=x})=\inf_{\gamma_x}\int||\hG-Y||d\gamma_x$, where the $\inf$ is taken over the set of all the couplings of $\hG(\eta,x)$ and $Y|X=x$. Adding a coordinate while preserving the norm, we have
\[
d_{W_1}(P_{\hat{G}(\eta,X=x)}, P_{Y|X=x})=\inf_{\gamma_x}\int\|(x,\hG)-(x,Y)\|d\gamma_x.
\]
Therefore,
\begin{align*}
\Ebb_{X} d_{W_1}(P_{\hat{G}(\eta,X)}, P_{Y|X})&=\int\inf_{\gamma_x}\int\|(x,\hG)-(x,Y)\|d\gamma_x dP_{X}\\
&\le \inf_{\pi} \int\|(X,\hG)-(X,Y)\|d\pi\\
&=d_{W_1}(P_{X,\hat{G}}, P_{X,Y}),
\end{align*}
where the last $\inf$ is taken over the set of all the couplings of $(X,\hG(\eta,X))$ and $(X,Y)$.
\end{proof}

Next, we prove Theorem \ref{thmlowdim}.

\begin{proof}[\textbf{of Theorem \ref{thmlowdim}}]
By taking $W_1 L_1=\ceil*{n^{\frac{d+q}{2d_A}}}$ and $R=M$ in Lemma \ref{lma2shen}, we get $\cE_1\precsim n^{-\frac{1}{d_A}}$. Since the range of $G_{\tht}$ covers all the $Y'_i$, we have $\cE_2=0$. By Assumption \ref{aspmin}, we have $\log \mathcal{N}(\epsilon, A, \|\cdot\|_2)\precsim \log (\frac{1}{\epsilon})^{d_A}$. Plugging it into \eqref{a1} and applying Lemma $\ref{ep}$ again by taking $\delta= n^{-\frac{1}{d_A}}$, we have $\Ebb\cE_3,\Ebb\cE_4\precsim n^{-\frac{1}{d_A}}$.
In all by Lemma \ref{decomp}, we have
\begin{align*}
    \Ebb d_{\mathcal{F}^1}(P_{X,\hat{G}},P_{X,Y})&\le 2\mathcal{E}_1+\mathcal{E}_2 +\Ebb\mathcal{E}_3+\Ebb \mathcal{E}_4\\
    &\precsim n^{-\frac{1}{d_A}} +0 +n^{-\frac{1}{d_A}}+n^{-\frac{1}{d_A}}\\
    & \precsim  n^{-\frac{1}{d_A}}.
\end{align*}
This completes the proof of Theorem \ref{thmlowdim}.
\end{proof}

Finally, we include Massart's lemma for convenience.
\begin{lemma}[Massart's Lemma]\label{mas}
Let $A\subseteq\Rbb^m$ be a finite set, with $r=\max_{x\in A}\|x\|_2$, then the following holds:
\[
\Ebb_\epsilon \frac{1}{n}\sup_{x\in A}\sum_{i=1}^n\epsilon_i x_i \le \frac{r\sqrt{2\log |A|}}{n},
\]
A proof of this lemma can be found at \cite[Theorem 3.3]{mohri2018foundations}.

\end{lemma}
\end{appendices}


\end{document}